\documentclass[a4paper,12pt,onecolumn]{IEEEtran}
\usepackage{enumerate}
\usepackage{setspace}
\usepackage{multirow}
\usepackage{color}
\ifCLASSOPTIONcompsoc
\usepackage[caption=false,font=normalsize,labelfon
t=sf,textfont=sf]{subfig}
\else
\usepackage[caption=false,font=footnotesize]{subfi
g}

\fi

\ifCLASSINFOpdf
\else
\fi
\usepackage{ulem}
\usepackage{amssymb}
\usepackage{graphicx}

\usepackage[colorlinks,citecolor=blue,linkcolor=blue,urlcolor=black,hyperindex]{hyperref}
\usepackage{url}
\usepackage{cite}
\usepackage{booktabs}
\usepackage{amsthm}
\usepackage{amsmath}
\newtheorem{theorem}{Theorem}

\newtheorem{remark}{Remark}
\newtheorem{proposition}{Proposition}

\theoremstyle{definition}
\newtheorem{definition}{Definition}

\newtheorem{example}{Example}

\allowdisplaybreaks

\begin{document}
%
\title{Managing Multi-Granular Linguistic Distribution Assessments in Large-Scale Multi-Attribute Group Decision Making}

%

\author{Zhen Zhang \IEEEmembership{Member,~IEEE},  Chonghui Guo and Luis Mart\'{\i}nez \IEEEmembership{Member,~IEEE} 
\thanks{Manuscript received XX-XX-XXXX; revised XX-XX-XXXX; accepted XX-XX-XXXX. This work was partly supported by the National
Natural Science Foundation of China (Nos. 71501023, 71171030), the Funds for Creative Research Groups of China (No. 71421001), the China Postdoctoral Science Foundation (2015M570248) and the Fundamental Research Funds for the Central Universities (DUT15RC(3)003) and by the Research Project TIN2015-66524 and FEDER funds}
\thanks{Z. Zhang and C. Guo are with the Institute of Systems Engineering, Dalian University of Technology, Dalian 116024, China. E-mails:
\href{mailto:zhangzhen@mail.dlut.edu.cn}{zhangzhen@mail.dlut.edu.cn (Z. Zhang)}, \href{mailto:guochonghui@tsinghua.org.cn}{guochonghui@tsinghua.org.cn (C. Guo)}}
\thanks{L. Mart\'{\i}nez is with the Department of Computer Science, University of J\'{a}en, J\'{a}en 23071, Spain (e-mail: \href{mailto:luis.martinez@ujaen.es}{luis.martinez@ujaen.es}).}}

\IEEEpubid{0000--0000/00\$00.00~\copyright~2007 IEEE}


\maketitle

\begin{abstract}
Linguistic large-scale group decision making (LGDM) problems are more and more common nowadays. In such problems a large group of decision makers are involved in the decision process and elicit linguistic information that are usually assessed in different linguistic scales with diverse granularity because of decision makers' distinct knowledge and background. To keep maximum information in initial stages of the linguistic LGDM problems, the use of multi-granular linguistic distribution assessments seems a suitable choice, however to manage such multi-granular linguistic distribution assessments, it is necessary the development of a new linguistic computational approach. In this paper it is proposed a novel computational model based on the use of extended linguistic hierarchies, which not only can be used to operate with multi-granular linguistic distribution assessments, but also can provide interpretable linguistic results to decision makers. Based on this new linguistic computational model, an approach to linguistic large-scale multi-attribute group decision making is proposed and applied to a talent selection process in universities.

\end{abstract}


%
\IEEEpeerreviewmaketitle
\newcommand{\bea}{\begin{eqnarray}}
\newcommand{\eea}{\end{eqnarray}}
\newcommand{\ben}{\begin{equation}}
\newcommand{\een}{\end{equation}}
\newcommand{\add}{\color{blue}\uwave}
\newcommand{\delete}{\color{red}\sout}

\section{Introduction}\label{sec:1}

Group decision making (GDM) is a common activity occurring in human being's daily life. For a typical multi-attribute group decision making (MAGDM) problem, a group of decision makers are usually required to express their assessments over alternatives with regard to some predefined criteria. Afterwards, the evaluation information is aggregated to form a group opinion, based on which collective evaluation and a ranking of alternatives can be obtained \cite{Palomares13eswa,Pedrycz11book}. Current GDM problems demand quick solutions and decision makers may either doubt or have vague or uncertain knowledge about alternatives; hence they cannot express their assessments with exact numerical values. Consequently a more realistic approach may be to use linguistic assessments instead of numerical values \cite{Rodriguez13ins,Delgado98ins}. In literature, MAGDM problems involving uncertainty are usually dealt with linguistic modeling that implies computing with words (CW) processes to obtain accurate and easily understood results \cite{Zadeh96tfs,Martinez10ijcis}.


Despite a large amount of research conducted on GDM with linguistic information \cite{lu2007,Agell12ins,Dong13tfs,Meng13ins}, there are still some challenges that need to be tackled. One of them is how to deal with GDM problems with large groups under linguistic environment. For traditional GDM problems, only a few number of decision makers may take part in the decision process. In recent years, the increase of technological and societal demands has given birth to new paradigms and means of making large-scale group decisions (such as e-democracy and social networks) \cite{Palomares14tfs}. As a result, the large-scale GDM problems have received more and more attentions from scholars. Large-scale GDM  (LGDM) can be grouped into four categories, i.e., clustering methods in LGDM \cite{Zahir99ejor,Wang14eswa,Liu14infu}, consensus reaching processes in LGDM \cite{Palomares14tfs,Xu15dss}, LGDM methods \cite{Liu15infu,Liu15ejor} and LGDM support systems \cite{Carvalho11jucs,Palomares14kbs}.

For linguistic LGDM problems, one important challenge is how to represent the group's linguistic assessment, especially when anonymity is needed to protect the privacy of decision makers. It seems that the linguistic models and computational processes used in traditional linguistic GDM problems \cite{Rodriguez13ijgs} can be directly extended to linguistic LGDM problems, which may include the linguistic aggregation operator-based approach and the models based on uncertain linguistic terms \cite{Zhang12kbs,Xu14gdn} and hesitant fuzzy linguistic term sets \cite{Rodriguez12tfs,Liao14ins,Liu14ins}. However, in linguistic LGDM problems, the group's assessments usually tend to present a distribution concerning the terms in the linguistic term set used, which can reflect the tendencies of preference from decision makers and provide more information about the collective assessments of alternatives. The linguistic models and computational processes introduced to deal with linguistic information in traditional linguistic GDM problems could imply an oversimplification of the elicited information from the very beginning, thus may lead to the loss and distortion of information. In order to keep the maximum information elicited by decision makers in a group in the initial stages of the decision process, this paper proposes the use of linguistic distribution assessments \cite{Zhang14infu,Dong15kbs} to represent group's linguistic information for linguistic LGDM problems.

Additionally in linguistic GDM problems, multiple sources of information with different degree of knowledge and background may take part in the decision process, which usually implies the appearance and the necessity of multiple linguistic scales (multi-granular linguistic information) to model properly different knowledge elicited by each source of information \cite{Morente-Molinera15kbs}. Different approaches have been introduced in literature not only to model and manage such a type of information but also for computing with it \cite{Herrera00fss,Herrera01tsmcb,Huynh05tsmcb,Chen06caie,Espinilla11ci,Fan10eswa,Dong15gdn}. Therefore in a linguistic LGDM problem, decision makers may use different linguistic term sets to provide the assessments over alternatives. In order to keep maximum information in initial stages of the decision process, the linguistic distribution assessments will be multi-granular linguistic ones. Hence, there is a clear need of dealing with multi-granular linguistic distribution assessments in the decision processes. Moreover, according to the CW scheme \cite{Martinez12ins,Rodriguez13ijgs}, it is also crucial to obtain interpretable final linguistic results to decision makers. Therefore, new models for representing and managing multi-granular linguistic distribution assessments will be developed.

Consequently, the aim of this paper is to introduce a new linguistic computational model which is able to deal with multi-granular linguistic information by keeping the maximum information at the initial stages, removing initial aggregation processes and modeling the information provided by experts with the use of linguistic distribution assessments to obtain a solution set of alternatives by a classical decision approach with specific operators defined for linguistic distribution assessments providing interpretable results.

The remainder of this paper is organized as follows. In Section \ref{sec:2}, some necessary preliminaries for the proposed model are presented. In Section \ref{sec:3}, improved distance measures and ranking approach of linguistic distribution assessments are provided. In Section \ref{sec:4}, a  new linguistic computational model is introduced to deal with multi-granular linguistic distribution assessments. In Section \ref{sec:5}, an approach is developed to deal with large-scale linguistic MAGDM problems using multi-granular linguistic distribution assessments. In Section \ref{sec:6}, an example is given to illustrate the proposed MAGDM approach. Finally, this paper is concluded in section \ref{sec:7}.

\section{Preliminaries}\label{sec:2}
In order to make this paper as self-contained as possible, some related preliminaries are presented in this section. In subsection \ref{sec:2-1}, we review some basic knowledge related to linguistic information and decision making. In subsection \ref{sec:2-2}, how to deal with multi-granular linguistic information is presented. In subsection \ref{sec:2-3}, related concepts about linguistic distribution assessments are provided.
\subsection{Linguistic information and decision making}\label{sec:2-1}

Many aspects of decision making activities in the real world are usually assessed in a qualitative way due to the vague or imprecise knowledge of decision makers. In such cases, the use of linguistic information seems to be a better way for decision makers to express their assessments. To manage linguistic information in decision making, linguistic modeling techniques are needed. In linguistic modeling, the linguistic variable defined by Zadeh \cite{Zadeh75ins1,Zadeh75ins2,Zadeh75ins3} is usually employed to reduce the communication gap between humans and computers. A linguistic variable is a variable whose values are not numbers but words in a natural or artificial language.

To facilitate the assessment process in linguistic decision making, a linguistic term set and its semantics should be chosen in advance. One way to generate the linguistic term set is to consider all the linguistic terms distributed on a scale in a total order \cite{Yager95ijar}. The most widely used linguistic term set is the one which has an odd value of granularity, being triangular-shaped, symmetrical and uniformly distributed its membership functions. A formal description of a linguistic term set can be given below.

Let $S=\{s_0,s_1,\ldots,s_{g-1}\}$ denote a linguistic term set with odd cardinality, the element $s_i$ represents the $i^{th}$ linguistic term in $S$, and $g$ is the cardinality of the linguistic term set $S$. Moreover, for the linguistic term set $S$, it is usually assumed that the midterm represents an assessment of ``approximately 0.5'', with the rest of the terms being placed uniformly and symmetrically around it. Moreover, $S$ should satisfy the following characteristics \cite{Herrera00tfs,Martinez10ijcis}: (1) The set is ordered: $s_i > s_j, \textrm{if}\  i > j$; (2) There is a negation operator: $\textrm{Neg}(s_i)=s_j$, such that $j=g-1-i$; (3) Maximization operator: $\textrm{max}(s_i, s_j) = s_i$, if $s_i \geqslant s_j$; (4) Minimization operator: $\textrm{min}(s_i, s_j) = s_i$, if $s_i \leqslant s_j$.

Different linguistic computational models have been developed for CW \cite{Martinez10ijcis,Rodriguez13ijgs}, such as models based on fuzzy membership functions \cite{Degani88ijar}, symbolic models based on ordinal scales \cite{Xu04ins}, models based on type-2 fuzzy sets \cite{Mendel06tfs}. However, such models sometimes may lead to information loss or lack interpretability. To enhance the accuracy and interpretability of linguistic computational models, Herrera and Mart\'{\i}nez \cite{Herrera00tfs} proposed the 2-tuple linguistic representation model, which is defined as below.
\begin{definition}\label{def-1}
\cite{Herrera00tfs} Let $S=\{s_0,s_1,\ldots,s_{g-1}\}$ be a linguistic term set and $\kappa\in[0,g-1]$ be a value representing the result of a symbolic aggregation operation, then the 2-tuple that expresses the equivalent information to $\kappa$ is obtained with the following function:
\bea\label{delta}
\begin{split}
&\Delta: [0,g-1]\rightarrow S\times [-0.5,0.5)\\
&\Delta(\kappa)=(s_k,\alpha),\\
\end{split}
\eea
with $k={\rm round}(\kappa)$, $\alpha=\kappa-k$, where ``round($\cdot$)" is the usual round operation, $s_k$ has the closest index label to $\kappa$, and $\alpha$ is the value of symbolic translation.
\end{definition}

\begin{definition}\label{def-2}
\cite{Herrera00tfs} Let $S=\{s_0,s_1,\ldots,s_{g-1}\}$ be a linguistic term set and $(s_k,\alpha)$ be a 2-tuple, there exists a function $\Delta^{-1}$, which can transform a 2-tuple into its equivalent numerical value $\kappa\in[0,g-1]$. The transformation function is defined as
\bea
\begin{split}
&\Delta^{-1}: S\times[-0.5,0.5)\rightarrow [0,g-1]\\
&\Delta^{-1}(s_k,\alpha)=k+\alpha=\kappa.\\
\end{split}
\eea
\end{definition}

Based on the above definitions, a linguistic term can be considered as a linguistic 2-tuple by adding the value 0 to it as a symbolic translation, i.e. $s_k\in S \Rightarrow (s_k,0)$. In this paper, the 2-tuple linguistic model will be used as the basic linguistic computational model.

\subsection{Multi-granular linguistic information}\label{sec:2-2}

When multiple decision makers or multiple criteria are involved in a linguistic decision making problem, the assessments concerning the alternatives are usually in the form of multi-granular linguistic information, which is due to the fact that a decision maker who wants to provide precise information may use a linguistic term set with a finer granularity, while a decision maker who is not able to be very precise about a certain domain may choose a linguistic term set with a coarse granularity \cite{Herrera00fss,Morente-Molinera16ins}.

To manage multi-granular linguistic information, different linguistic computational models have been proposed, including models based on fuzzy membership
functions \cite{Jiang08ins,Zhang12kbs}, ordinal models based on a basic linguistic term set \cite{Chen06caie,Herrera00fss,Xu09gdn}, the linguistic hierarchies (LH) model \cite{Herrera01tsmcb}, ordinal models based on hierarchical trees \cite{Huynh05tsmcb}, models based qualitative description spaces \cite{Roselo14infu} and ordinal models based discrete fuzzy numbers \cite{Massanet14ins}. For a systematic review about multi-granular fuzzy linguistic modeling, the readers can refer to \cite{Morente-Molinera15kbs}.

To fuse linguistic information with any linguistic scale, Espinilla \emph{et al.} \cite{Espinilla11ci} introduced an extended linguistic hierarchies (ELH) model based on the LH model. In this paper, the ELH model will be used to handle multi-granular linguistic information. Before introducing the ELH model, we first recall the LH model proposed by Herrera and Mart\'{\i}nez \cite{Herrera01tsmcb}.

A LH is the union of all levels $i$: $LH=\bigcup_i l(i,g(i))$, where each level $i$ of a LH corresponds to a linguistic term set with a granularity of $g(i)$ denoted as: $S^{g(i)}=\{s_0^{g(i)},s_1^{g(i)},\ldots,s_{g(i)-1}^{g(i)}\}$, and a linguistic term set of level $i+1$ is obtained from its predecessor as $l(i,g(i))\rightarrow l(i+1, 2\cdot g(i)-1)$. Based on the LH basic rules, a transformation function $TF_{i'}^{i}$ between any two linguistic levels $i$ and $i'$ of the LH is defined as below.
\begin{definition}\label{LH}
\cite{Herrera01tsmcb} Let $LH=\bigcup_i l(i,g(i))$ be a LH whose linguistic term sets are denoted as $S^{g(i)}=\{s_0^{g(i)},s_1^{g(i)},\ldots,s_{g(i)-1}^{g(i)}\}$, and let us consider the 2-tuple linguistic representation. The transformation function from a linguistic label in level $i$ to a label in level $i'$, satisfying the LH basic rules, is defined as
\bea
TF^{i}_{i'}(s_k^{g(i)},\alpha^{g(i)})=\Delta\left(\dfrac{\Delta^{-1}(s_k^{g(i)},\alpha^{g(i)})\cdot (g(i')-1)}{g(i)-1}\right).
\eea
\end{definition}

The ELH model constructs extended linguistic hierarchies based on the following proposition.

\begin{proposition}\label{multi1}
\cite{Espinilla11ci} Let $\{S^{g(1)},S^{g(2)},\ldots,S^{g(n)}\}$ be a set of linguistic term sets, where the granularity $g(i)$ is an odd value, $i=1,2,\ldots,n$. A new linguistic term set $S^{g(i^*)}$ with $i^*=n+1$ that keeps all the formal modal points of the $n$ linguistic term sets has the minimal granularity:
\begin{equation}\label{lcm}
g(i^*)=LCM(\delta_1,\delta_2,\ldots,\delta_n)+1,
\end{equation}
where $LCM$ is the least common multiple and $\delta_i=g(i)-1$, $i=1,2,\ldots,n$. The set of former modal points of the level $i$ is defined as $FP_{i}=\{fp^0_{i},\ldots,fp^j_{i},\ldots,fp^{2\cdot\delta_{i}}_{i}\}$ and each former modal point $fp^j_{i} \in [0,1]$ is located at $fp^j_{i}=\frac{j}{2\cdot\delta_{i}}$.
\end{proposition}

Based on Proposition \ref{multi1}, an ELH which is the union of the $n$ levels required by the experts and the new level $l(i^*, g(i^*))$ that keeps all the former modal points to provide accuracy in the processes of CW is denoted by
\begin{equation}\label{ELH}
ELH=\bigcup_{i=1}^{n+1} l(i,g(i)).
\end{equation}

Espinilla \emph{et al.} \cite{Espinilla11ci} defined a transformation function which can transform any pair of linguistic term sets, $i$ and $i'$, in the ELH without loss of information. The basic idea of the transformation function is as follows. First, transform linguistic terms at any level $l(i, g(i))$ in the ELH into those at $l(i^*, g(i^*))$, being $i^*=n+1$, that keeps all the former modal points of the level $i$, by means of $TF^{i}_{i^*}$ without loss of information, and then transform the linguistic terms at $l(i^*, g(i^*))$ in the ELH into any level $l(i', g(i'))$ by means of $TF_{i'}^{i^*}$ without loss of information.

\begin{definition}
 \cite{Espinilla11ci} Assume $i$ and $i'$ be any pair of linguistic term sets in the ELH and $i^*$ is the level $l(n+1, g(n + 1))$ in the ELH, the new extended transformation function $ETF_{i'}^i$ is defined as
\begin{equation}
\begin{split}
&ETF_{i'}^i: l(i, g(i))\rightarrow l(i', g(i'))\\
&ETF_{i'}^i=TF^{i}_{i^*}\circ TF_{i'}^{i^*},
\end{split}
\end{equation}
where $TF^{i}_{i^*}$ and $TF_{i'}^{i^*}$ are the transformation functions as defined in the LH model.
\end{definition}

\subsection{Linguistic distribution assessments}\label{sec:2-3}

In this subsection, some related concepts of linguistic distribution assessments are presented. First, the definition of a linguistic distribution assessment is revised.
\begin{definition}\label{def-4}
\cite{Zhang14infu} Let $S=\{s_0,s_1,\ldots,s_{g-1}\}$ denote a linguistic term set and $\beta_k$ be the symbolic proportion of $s_k$, where $s_k\in S$, $\beta_k\geqslant 0$, $k=0,1,\ldots,g-1$ and $\sum_{k=0}^{g-1}\beta_k=1$, then an assessment $m=\{\langle s_k,\beta_k\rangle|k=0,1,\ldots,g-1\}$ is called a linguistic distribution assessment of $S$, and the expectation of $m$ is defined as a linguistic 2-tuple by $E(m)$, where $E(m)=\Delta\left(\sum\limits_{k=0}^{g-1}k\beta_k\right)$ \footnote{The representation is different from the definition provided in  \cite{Zhang14infu}, but they have the same meaning.}. For two linguistic distribution assessments $m_1$ and $m_2$, if $E(m_1)\geqslant E(m_2)$, then $m_1\geqslant m_2$.
\end{definition}

\begin{remark}
A linguistic distribution assessment can be used to represent the linguistic assessment of a group. Assume that the originality of a research project was assessed by five experts using linguistic terms from a linguistic term set $S=\{s_0,s_1,\ldots,s_4\}$. If the assessments of the five experts were $s_1,s_2,s_1,s_3,s_2$, then the overall assessment could be denoted as a linguistic distribution assessment $\{\langle s_1,0.4\rangle,\langle s_2,0.4\rangle,\langle s_3,0.2\rangle\}$. Based on a linguistic distribution assessment, we not only can roughly know the possible assessment of an alternative in a linguistic way, but also can derive the distribution of each linguistic term, which keeps the maximum information elicited by decision makers in a group.
\end{remark}

Zhang \emph{et al.} \cite{Zhang14infu} developed the weighted averaging operator of linguistic distribution assessments (i.e., DAWA operator), which is defined as follows.
\begin{definition}\label{DAWA}
\cite{Zhang14infu} Let $\{m_1,m_2,\ldots,m_n\}$ be a set of linguistic distribution assessments of $S$, where $m_i=\{\langle s_k,\beta_k^i \rangle |k=0,1,\ldots,g-1\}$, $i=1,2,\ldots,n$, and $w=(w_1,w_2,\ldots,w_n)^{\rm T}$ be an associated weighting vector that satisfies $w_i\geqslant0$ and $\sum_{i=1}^{n}w_i=1$, then the weighted averaging operator of $\{m_1,m_2,\ldots,m_n\}$ is defined as
\bea
{\rm DAWA}_w(m_1,m_2,\ldots,m_n)=\{\langle s_k,\beta_k \rangle| k=0,1,\ldots,g-1\},
\eea
where $\beta_k=\sum_{i=1}^{n}w_i\beta_{k}^{i}$, $k=0,1,\ldots,g-1$.
\end{definition}

The distance measure between two linguistic distribution assessments is also given in \cite{Zhang14infu}, as showed below.
\begin{definition}\label{Dis1}
\cite{Zhang14infu} Let $m_1=\{\langle s_k,\beta_k^1\rangle| k=0,1,\ldots,g-1\}$ and $m_2=\{\langle s_k,\beta_k^2\rangle| k=0,1,\ldots,g-1\}$ be two linguistic distribution assessments of a linguistic term set $S$, then the distance between $m_1$ and $m_2$ is defined as
\bea\label{Dis11}
d(m_1,m_2)=\frac{1}{2}\sum\limits_{k=0}^{g-1}|\beta_k^1-\beta_k^2|.
\eea
\end{definition}

\section{Improving distance and ranking methods for linguistic distribution assessments}\label{sec:3}

In this section, it is pointed out that previous distance measure and ranking method for linguistic distribution assessments present some flaws, and a new distance measure and a new ranking method are then introduced to overcome such flaws. First, it is showed the flaws of the distance measure defined in \cite{Zhang14infu}, i.e. Definition \ref{Dis1}, with Example \ref{ex2}.
\begin{example}\label{ex2}
Let $S^{\textrm{example}}=\{s_0,s_1,\ldots,s_{4}\}$ be a linguistic term set and there are three linguistic distribution assessments: $m_1=\{\langle s_0,0\rangle,\langle s_1,1\rangle,\langle s_2,0\rangle,\langle s_3,0\rangle,\langle s_4,0\rangle\}$, $m_2=\{\langle s_0,1\rangle,\langle s_1,0\rangle,\langle s_2,0\rangle,\langle s_3,0\rangle,\langle s_4,0\rangle\}$ and $m_3=\{\langle s_0,0\rangle, \langle s_1,0\rangle,\langle s_2,0\rangle,\langle s_3,0\rangle,\langle s_4,1\rangle\}$. By the definition of the linguistic distribution assessment, we know that a linguistic term $s_i$ of $S$ is a special case of the linguistic distribution assessment $m=\{\langle s_k,\beta_k\rangle|k=0,1,\ldots,g-1\}$ with $\beta_i=1$ and $\beta_k=0$, for all $k\neq i$, i.e. $m_1=s_1$, $m_2=s_0$ and $m_3=s_4$. However, by Definition \ref{Dis1}, we can obtain
\[d(m_1,m_2)=\frac{1}{2}(1+1+0+0+0)=1,\ d(m_1,m_3)=\frac{1}{2}(0+1+0+0+1)=1,\]
which means that the distance between $s_1$ and $s_0$ is equal to that between $s_1$ and $s_4$. Obviously it is unreasonable.
\end{example}

From Definition \ref{Dis1}, it can be seen that (\ref{Dis11}) just calculates the deviation between symbolic proportions and ignores the importance of linguistic terms. In this paper a novel distance measure between two linguistic distribution assessments is defined as:
\begin{definition}\label{Dis2}
 Let $m_1=\{\langle s_k,\beta_k^1\rangle| k=0,1,\ldots,g-1\}$ and $m_2=\{\langle s_k,\beta_k^2\rangle| k=0,1,\ldots,g-1\}$ be two linguistic distribution assessments of a linguistic term set $S$, then the distance between $m_1$ and $m_2$ is defined as
\bea\label{Dis21}
d(m_1,m_2)=\frac{1}{g-1}\left|\sum\limits_{k=0}^{g-1}(\beta_k^1-\beta_k^2)k\right|.
\eea
\end{definition}

Reconsider Example \ref{ex2}. By Definition \ref{Dis2} is calculated $d(m_1,m_2)=0.25,\ d(m_1,m_3)=0.75$, which is more reasonable to the intuition.

Looking now at the ranking problem of a linguistic distribution assessments collection. Zhang {\itshape et al.} \cite{Zhang14infu} utilized the expectation values to rank linguistic distribution assessments. However, there may be cases that the expectation values of some linguistic distribution assessments are equal. As a result, the comparison rule mentioned in Definition \ref{def-4} sometimes cannot distinguish these linguistic distribution assessments. As the uncertainty in the sense of inaccuracy of a linguistic distribution assessment is reflected by its distribution, which can be measured by using Shannon's entropy\cite{Cover06book}. It is then proposed that the ranking of linguistic distribution assessments will be computed by an inaccuracy function for linguistic distribution assessments and several comparison rules introduced below.
\begin{definition}\label{inacc}
Let $m=\{\langle {{s}_{k}},{{\beta }_{k}}\rangle |k=0,1,\ldots ,g-1\}$ be a linguistic distribution assessment of a linguistic term set $S=\{{{s}_{0}},{{s}_{1}},\ldots ,{{s}_{g-1}}\}$, where $s_k\in S$, ${{\beta }_{k}}\geqslant 0$, $k=0,1,\ldots ,g-1$ and $\sum\nolimits_{k=0}^{g-1}{{{\beta }_{k}}}=1$. The inaccuracy function of $m$ is defined as $T(m)=-\sum\nolimits_{k=0}^{g-1}{{{\beta }_{k}}{{\log }_{2}}}{{\beta }_{k}}$ \footnote{$0\log_2 0=0$ is defined in this paper.}.
\end{definition}

\begin{definition}\label{def-6}
Let $m_1$ and $m_2$ be two linguistic distribution assessments, then the comparison rules are defined as follows:
(1)	If $E(m_1)>E(m_2)$, then $m_1>m_2$;
(2)	If $E(m_1)=E(m_2)$ and $T(m_1)<T(m_2)$, then $m_1>m_2$; If $E(m_1)=E(m_2)$ and $T(m_1)=T(m_2)$, then $m_1=m_2$.
\end{definition}

\begin{example}
Let $S^{\textrm{example}}=\{s_0,s_1,\ldots,s_{4}\}$ be a linguistic term set and there are three linguistic distribution assessments: $m_1=\{\langle s_1,0.3\rangle,\langle s_2,0.4\rangle,\langle s_3,0.3\rangle\}$, $m_2=\{\langle s_2,1\rangle\}$ and $m_3=\{\langle s_1,0.3\rangle, \langle s_2,0.7\rangle\}$.

By Definition \ref{def-4}, $E(m_1)=(s_2,0)$, $E(m_2)=(s_2,0)$, $E(m_3)=(s_2,-0.3)$. Therefore, $m_1=m_2>m_3$. However, if the inaccuracy function values of the three linguistic distribution assessments are, $T(m_1)=1.5710$, $T(m_2)=0$, $T(m_3)=0.8813$. According to Definition \ref{def-6}, it follows that $m_2>m_1>m_3$. Obviously, the new comparison rules can distinguish linguistic distribution assessments more effectively.
\end{example}

\section{Dealing with multi-granular linguistic distribution assessments}\label{sec:4}
As the focus of this paper is to deal with LGDM problems with multi-granular linguistic information, this section is devoted to develop a new computational model to deal with multi-granular linguistic distribution assessments. Due to the fact that our proposal for dealing with multi-granular linguistic distribution assessments and obtaining interpretable results will be based on tools introduced for linguistic 2-tuple values, Subsection \ref{sec:4-1} shows how to transform a linguistic 2-tuple into a linguistic distribution assessment. Afterwards, a new model for managing multi-granular linguistic distribution assessments is developed in Subsection \ref{sec:4-2}.

\subsection{Transforming a linguistic 2-tuple into a linguistic distribution assessment}\label{sec:4-1}

This subsection discusses the relationship between a linguistic 2-tuple and a linguistic distribution assessment. For convenience, let $S=\{s_0,s_1,\ldots,s_{g-1}\}$ be a linguistic term set as defined in section \ref{sec:2} and $(s_k,\alpha)$ be a linguistic 2-tuple, then:

(1) If $\alpha>0$, $(s_k,\alpha)$ denotes the linguistic information between $s_k$ and $s_{k+1}$.

(2) If $\alpha<0$, $(s_k,\alpha)$ denotes the linguistic information between $s_{k-1}$ and $s_{k}$.

(3) If $\alpha=0$, $(s_k,\alpha)$ denotes the linguistic information $s_{k}$.

\begin{proposition}\label{pro0}
Let $l$ be the integer part of $\kappa=\Delta^{-1}(s_k,\alpha)$, then a linguistic 2-tuple $(s_k,\alpha)$ denotes the linguistic information between $s_l$ and $s_{l+1}$ if $\alpha\neq 0$.
\end{proposition}

\begin{proof}
We consider two cases.

Case 1: $\alpha>0$. In this case, $\Delta^{-1}(s_k,\alpha)=k+\alpha>k$. Hence, $l=k$ and $l+1=k+1$.

Case 2: $\alpha<0$. In this case, $\Delta^{-1}(s_k,\alpha)=k+\alpha<k$. Hence, $l=k-1$ and $l+1=k$.

According to the previous results, a linguistic 2-tuple $(s_k,\alpha)$ denotes the linguistic information between $s_l$ and $s_{l+1}$ if $\alpha\neq 0$. This completes the proof of Proposition \ref{pro0}.
\end{proof}

Proposition \ref{pro0} demonstrates that a linguistic 2-tuple $(s_k,\alpha)$, $(\alpha \neq 0)$ can denote the linguistic information between two successive linguistic terms $s_l$ and $s_{l+1}$. From the perspective of linguistic distribution assessments, the linguistic information between $s_l$ and $s_{l+1}$ should be denoted as a linguistic distribution assessment $m=\{\langle s_l,1-\beta\rangle,\langle s_{l+1},\beta\rangle\}$. It is then necessary to determine the value of $\beta$.

As the linguistic information between $(s_k,\alpha)$ and $m$ is equivalent, the expectation of $m$ should be equal to $(s_k,\alpha)$. Therefore,  $\Delta(l\times(1-\beta)+(l+1)\times\beta)=(s_k,\alpha)$, i.e.
\begin{equation}\label{eq}
l\times(1-\beta)+(l+1)\times\beta=\Delta^{-1}(s_k,\alpha).
\end{equation}

By solving (\ref{eq}), $\beta=\Delta^{-1}(s_k,\alpha)-l$.

From the previous analysis, a linguistic 2-tuple $(s_k,\alpha)$, $(\alpha \neq 0)$ can be denoted as a linguistic distribution assessment $m=\{\langle s_l,1-\beta\rangle,\langle s_{l+1},\beta\rangle\}$, where $l$ is the integer part of $\Delta^{-1}(s_k,\alpha)$ and $\beta=\Delta^{-1}(s_k,\alpha)-l$.

It is easy to verify that the above statement also holds for the case $\alpha=0$, i.e. a linguistic 2-tuple $(s_k,0)$ can be denoted as a linguistic distribution assessment $m=\{\langle s_l,1-\beta\rangle,\langle s_{l+1},\beta\rangle\}$, where $l=k$ and $\beta=0$.
\begin{definition}\label{def-k}
Let $S=\{s_0,s_1,\ldots,s_{g-1}\}$ be a linguistic term set and $\Omega$ be the set of all the linguistic distribution assessments of $S$,  and there exists a function $F$, which can transform a linguistic 2-tuple $(s_k,\alpha)$ into its equivalent linguistic distribution assessment. The transformation function is defined as
\bea\label{eq:tran}
\begin{split}
&F: S\times[-0.5,0.5)\rightarrow \Omega\\
&F(s_k,\alpha)=\{\langle s_l,1-\beta\rangle,\langle s_{l+1},\beta\rangle\},\\
\end{split}
\eea
where $l$ is the integer part of $\Delta^{-1}(s_k,\alpha)$ and $\beta=\Delta^{-1}(s_k,\alpha)-l$.
\end{definition}

For Definition \ref{def-k}, the following theorem is given:
\begin{theorem}\label{thm00}
Let $S=\{s_0,s_1,\ldots,s_{g-1}\}$ be a linguistic term set. The equivalent linguistic distribution assessment of a linguistic 2-tuple $(s_k,\alpha)$ is
\bea
F(s_k,\alpha)=\left\{\begin{array}{l}\begin{aligned}
&\{\langle s_k,1-\alpha\rangle,\langle s_{k+1},\alpha\rangle\}& {\rm if}\ \alpha\geqslant 0,\\
&\{\langle s_{k-1},-\alpha\rangle,\langle s_{k},1+\alpha\rangle\}& {\rm if}\ \alpha<0. \\
\end{aligned}
\end{array}\right.
\eea
\end{theorem}

\begin{proof}
If $\alpha\geqslant 0$, $\Delta^{-1}(s_k,\alpha)\geqslant k$, then $l=k$ and $\beta=\Delta^{-1}(s_k,\alpha)-l=k+\alpha-k=\alpha$. By (\ref{eq:tran}), $F(s_k,\alpha)=\{\langle s_k,1-\alpha\rangle,\langle s_{k+1},\alpha\rangle\}$.

If $\alpha< 0$, $\Delta^{-1}(s_k,\alpha)<k$, then $l=k-1$ and $\beta=\Delta^{-1}(s_k,\alpha)-l=k+\alpha-k+1=1+\alpha$. By (\ref{eq:tran}), $F(s_k,\alpha)=\{\langle s_{k-1},-\alpha\rangle,\langle s_{k},1+\alpha\rangle\}$.

This completes the proof of Theorem \ref{thm00}.
\end{proof}

Definition \ref{def-k} and Theorem \ref{thm00} establish the relationship between a linguistic 2-tuple and a linguistic distribution assessment, which will be helpful in the following section.

\begin{example}\label{ex3}
Let $S^{\textrm{example}}=\{s_0,s_1,\ldots,s_{4}\}$ be a linguistic term set, $(s_2, 0.6)$ and $(s_2,-0.3)$ be two linguistic 2-tuples. Based on Theorem \ref{thm00},  $F(s_2,0.6)=\{\langle s_2,1-0.6\rangle,\langle s_{2+1},0.6\rangle\}=\{\langle s_2,0.4\rangle,\langle s_{3},0.6\rangle\}$ and $F(s_2,-0.3)=\{\langle s_{2-1},0.3\rangle,\langle s_{2},1-0.3\rangle\}=\{\langle s_1,0.3\rangle,\langle s_{2},0.7\rangle\}$.
\end{example}

\subsection{Unifying multi-granular linguistic distribution assessments}\label{sec:4-2}

To deal with decision making problems with multi-granular linguistic information, a natural solution is to unify them and derive linguistic information based on the same linguistic term set \cite{Herrera00fss,Herrera01tsmcb}. Afterwards, the multi-granular linguistic information can be fused. This subsection focuses on the unification of multi-granular linguistic distribution assessments.

For convenience, some notations are defined as follows. Let $\{S^{g(1)},S^{g(2)},\ldots,S^{g(n)}\}$ be a set of linguistic term sets, where $S^{g(i)}=\{s_0^{g(i)},s_1^{g(i)},\ldots, s_{g(i)-1}^{g(i)}\}$ is a linguistic term set with an odd granularity $g(i)$, $i=1,2,\ldots,n$, and an ELH is constructed by Eq. (\ref{ELH}) as $ELH=\bigcup_{i=1}^{n+1} l(i,g(i))$, where $i^*=n+1$ is the level of $l(i^*,g(i^*))$. By Proposition \ref{multi1}, $g(i^*)=LCM(\delta_1,\delta_2,\ldots,\delta_n)+1$, where $\delta_i=g(i)-1$, $i=1,2,\ldots,n$. For the level $i^*$, the linguistic term set is denoted by $S^{g(i^*)}=\{s_0^{g(i^*)},s_1^{g(i^*)},\ldots,s_{g(i^*)-1}^{g(i^*)}\}$. Moreover, a linguistic distribution assessment on a linguistic term set is denoted as $S^{g(i)}$ by $m^{g(i)}=\{\langle s_k^{g(i)},\beta_k^i\rangle|k=0,1,\ldots,g(i)-1\}$.

Now, it is necessary to transform a linguistic distribution assessment $m^{g(i)}$ into a linguistic distribution assessment on another linguistic term set $S^{g(i')}$, where $i'=1,2,\ldots,n$ and $i'\neq i$.

Motivated by the extended transformation function of the ELH model, it is proposed a two-stage procedure to conduct the transformation process.

\textbf{Stage 1:} Transform the linguistic distribution assessment $m^{g(i)}$ into a linguistic distribution assessment on $S^{g(i^*)}$.

\textbf{Stage 2:} Transform the linguistic distribution assessment on $S^{g(i^*)}$ into a linguistic distribution assessment on $S^{g(i')}$.

 Looking at Stage 1, intuitively, it can be first transformed the linguistic terms in $S^{g(i)}$ into linguistic information in the linguistic term set $S^{g(i^*)}$ by the function $TF^{i}_{i^*}$. As the transformation is from a low level to a high level, the transformed linguistic information are normative linguistic terms without symbolic translations. As a result, it is only necessary to attach corresponding symbolic proportions in $S^{g(i)}$ with each linguistic term in $S^{g(i^*)}$. By doing so, a linguistic distribution assessment on $S^{g(i^*)}$ is derived. Formally, it is given the following definition.
\begin{definition}\label{multi2}
Let $\{S^{g(1)},S^{g(2)},\ldots,S^{g(n)}\}$ and $m^{g(i)}$ be defined as before, then $m^{g(i)}$ can be transformed into a linguistic distribution assessment on $S^{g(i^*)}$ by
\begin{equation}
m^{g(i^*)}=\{\langle s_k^{g(i^*)},\gamma_k^{i}\rangle|k=0,1,\ldots,g(i^*)-1\}
\end{equation}
with
\ben\label{eqn:14}
\gamma_k^{i}=\left\{\begin{array}{l}
\begin{split}
&\beta^i_{l(i,k)} &{\rm if}\  l(i,k)&\in \{0,1,\ldots,g(i)-1\};\\
&0  &{\rm if}\  l(i,k)&\notin \{0,1,\ldots,g(i)-1\},\
\end{split}
\end{array}\right.
\een
where $l(i,k)=\dfrac{k*(g(i)-1)}{g(i^*)-1}$, $k=0,1,\ldots,g(i^*)-1$.
\end{definition}

The meaning of Definition \ref{multi2} is to find out the linguistic terms in $S^{g(i^*)}$, whose corresponding linguistic terms in $S^{g(i)}$ have non-zero symbolic proportions in $m^{g(i)}$, and then assign the non-zero symbolic proportions to them.

\begin{theorem}\label{thm0}
The transformed $m^{g(i^*)}$ is a linguistic distribution assessment of $S^{g(i^*)}$.
\end{theorem}

\begin{proof}
According to \cite{Espinilla11ci}, the transformation from $S^{g(i)}$ to $S^{g(i^*)}$ is one-to-one, i.e. each linguistic term of $S^{g(i)}$ corresponds to a linguistic term of $S^{g(i^*)}$. Specifically, we have
\begin{equation*}
s_0^{g(i)} \leftrightarrow s_0^{g(i^*)}, \ s_1^{g(i)} \leftrightarrow s_{\frac{g(i^*)-1}{g(i)-1}}^{g(i^*)},\ \ldots,\ s_{g(i)-1}^{g(i)} \leftrightarrow s_{g(i^*)-1}^{g(i^*)}.
\end{equation*}

If $k=0, \dfrac{g(i^*)-1}{g(i)-1},2\cdot \dfrac{g(i^*)-1}{g(i)-1},\ldots, g(i^*)-1$, then $l(i,k)=0,1,2,\ldots,g(i)-1$ and
$\gamma_k^i=\beta_0^i,\beta_1^i,\beta_2^i,\ldots,$\\$\beta_{g(i)-1}^i$. Thus we have $\sum_{l(i,k)\in \{0,1,\ldots,g(i)-1\}}\gamma_k^i=\sum_{k=0}^{g(i)-1}\beta_k^i=1$.

Since $\gamma_k^i=0$, $\forall l(i,k)\notin \{0,1,\ldots,g(i)-1\}$ it is obtained $\sum_{k=0}^{g(i^*)-1}\gamma_k^i=1$. Therefore, $m^{g(i^*)}$ is a linguistic distribution assessment of $S^{g(i^*)}$, which completes the proof of Theorem \ref{thm0}.
\end{proof}

At Stage 2 it is transformed the linguistic distribution assessment on $S^{g(i^*)}$ into a linguistic distribution assessment on $S^{g(i')}$.

At first glance, it might be thought that we can also utilize the transformation function $TF_{i'}^{i^*}$ to transform the linguistic terms in $S^{g(i^*)}$ into linguistic information of $S^{g(i')}$, i.e.
\bea\label{mi}
TF_{i}^{i^*}(s_k^{g(i^*)},0)=\Delta\left(\dfrac{k \cdot (g(i)-1)}{g(i^*)-1}\right), k=0,1,\ldots,g(i^*)-1,
\eea
and then attach the corresponding symbolic proportions. However, such transformation is from a high level to a low level. Hence, some linguistic terms in $S^{g(i^*)}$ may be transformed into linguistic 2-tuples of $S^{g(i')}$. In this way, the derived result is not a normative linguistic distribution assessment of $S^{g(i')}$.

To address this issue and according to Definition \ref{def-k}, a linguistic 2-tuple can be transformed into its equivalent linguistic distribution assessment by (\ref{eq:tran}). Therefore, it can be first transformed each linguistic 2-tuple derived by (\ref{mi}) into its equivalent linguistic distribution assessment by using Definition \ref{def-k} and obtain $g(i^*)$ linguistic distribution assessments of $S^{g(i')}$, i.e.
\bea\label{mik}
\begin{split}
m_k^{S^{g(i')}}&=F\left(TF_{i}^{i^*}(s_k^{g(i^*)},0)\right)\\
&=\{\langle s_{l'(i,k)}^{g(i)},1-\theta(i,k)\rangle,\langle s_{l'(i,k)+1}^{g(i)},\theta(i,k)\rangle\},\ k=0,1,\ldots, g(i^*),
\end{split}
\eea
 where $l'(i,k)$ is the integer part of $\dfrac{k \cdot (g(i)-1)}{g(i^*)-1}$ and $\theta(i,k)=\dfrac{k \cdot (g(i)-1)}{g(i^*)-1}-l'(i,k)$.

Considering the symbolic proportion of each $TF_{i}^{i^*}(s_k^{g(i^*)},0)$, it can be aggregated these linguistic distribution assessments $m_k^{S^{g(i')}}$, $k=0,1,\ldots,g(i^*)-1$ into a new one by the DAWA operator, which yields a linguistic distribution assessment of $S^{g(i')}$. Formally, it is provided the following definition.

\begin{definition}\label{multi3}
Let $\{S^{g(1)},S^{g(2)},\ldots,S^{g(n)}\}$ and $m^{g(i)}$ be defined as before, then $m^{g(i^*)}$ derived by Definition \ref{multi2} can be transformed into a linguistic distribution assessment on $S^{g(i')}$ by
\begin{equation}\label{mi1}
m^{g(i')}={\rm DAWA}_{\omega}\left(m_0^{S^{g(i')}},m_1^{S^{g(i')}},\ldots,m_{g(i^*)-1}^{S^{g(i')}}\right),
\end{equation}
where ${\omega}=(\gamma_0^i,\gamma_1^i,\ldots,\gamma_{g(i^*)-1}^i)^{\rm T}$ and $m_k^{S^{g(i')}}$, $k=0,1,\ldots,g(i^*)-1$ is calculated by (\ref{mik}).
\end{definition}

The procedures of the two-stage transformation are illustrated by Fig. \ref{fig1}.
\begin{figure}
  \centering
  \includegraphics[scale=1]{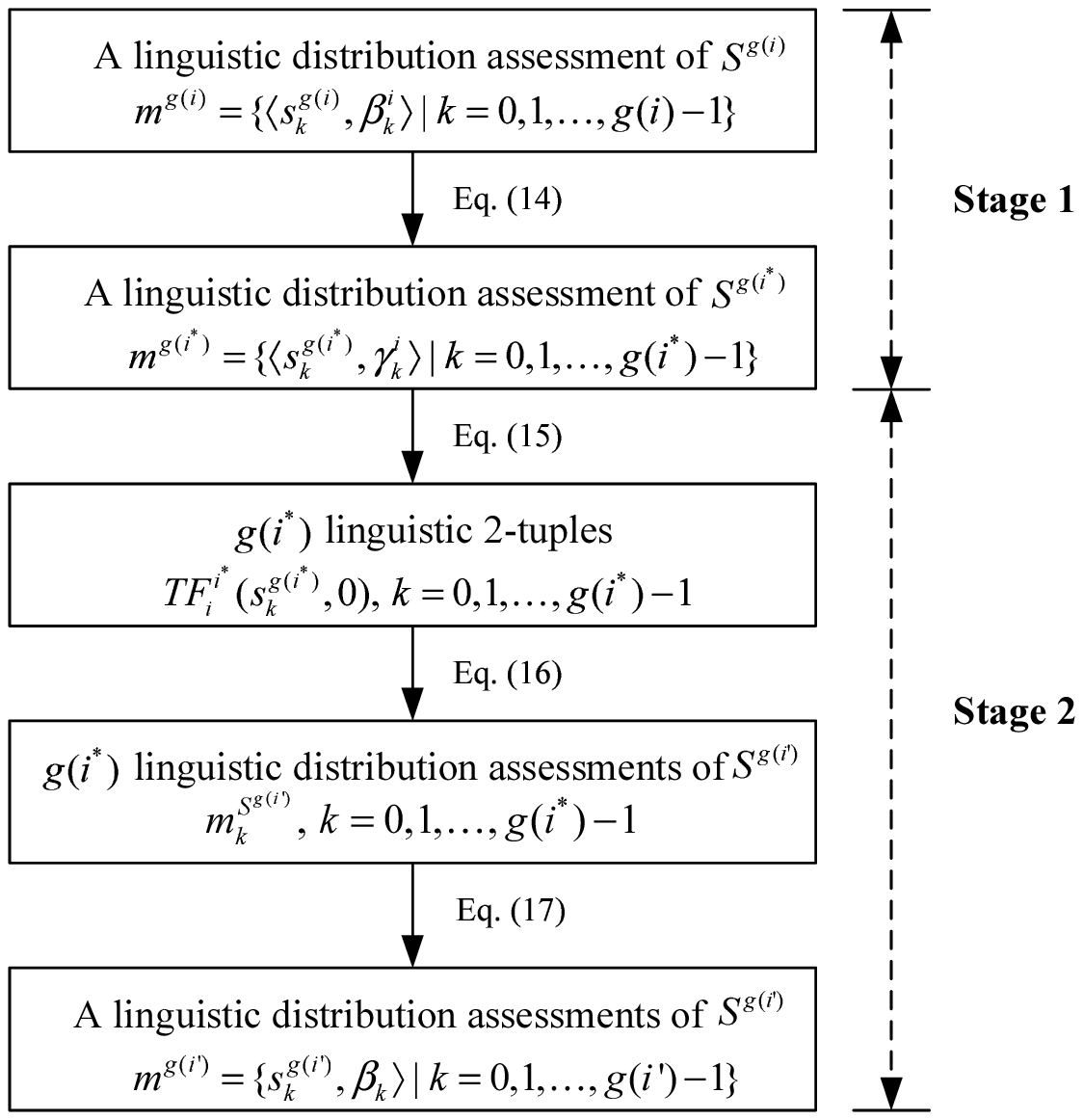}\\
  \caption{Flowchart of the two-stage transformation}\label{fig1}
\end{figure}

\begin{theorem}\label{thm1}
$m^{g(i')}$ derived by Definition \ref{multi3} is a linguistic distribution assessment.
\end{theorem}
\begin{proof}
Based on the above analysis, we have that each $m^{g(i')}$, $k=0,1,\ldots,g(i^*)-1$ is a linguistic distribution assessment. Moreover, $\sum_{k=0}^{g(i^*)-1}\gamma_k^i=1$. Since the weighted average of some linguistic distribution assessments is also a linguistic distribution assessment \cite{Zhang14infu}, $m^{g(i')}$ is a linguistic distribution assessment.
\end{proof}

\begin{example}\label{ex11}
Let $S^5=\{s_0^5,s_1^5,\ldots,s_4^5\}$ and $S^7=\{s_0^7,s_1^7,\ldots,s_6^7\}$ be two linguistic term sets, and there are two linguistic distribution assessments to be fused, i.e. $\{\langle s_1^5,0.3\rangle,\langle s_2^5,0.5\rangle,\langle s_3^5,0.2\rangle\}$ and $\{\langle s_1^7,0.25\rangle,\langle s_2^7,0.3\rangle,$\\$\langle s_3^7,0.45\rangle\}$.
\end{example}

Here, there are two linguistic term sets, i.e. $S^5$ and $S^7$. According to Proposition \ref{multi1}, we have $g(i^*)=LCM(4,6)=12$. Therefore, $S^{g(i^*)}=S^{13}=\{s_0^{13}, s_1^{13},\ldots, s_{12}^{13}\}$. For the first linguistic distribution assessment, since ${\bf 3}\times4/12={\bf 1}$, ${\bf6}\times4/12={\bf 2}$, ${\bf9}\times4/12={\bf 3}$, by (\ref{eqn:14}), then $\gamma_3^1=\beta^1_1=0.3$, $\gamma_6^1=\beta^1_2=0.5$, $\gamma_9^1=\beta^1_3=0.2$ and a linguistic distribution assessment $m^{13}_1=\{\langle s_3^{13},0.3\rangle,\langle s_6^{13},0.5\rangle,\langle s_9^{13},0.2\rangle\}$ is derived.

In what follows, we attempt to transform $m_1^{13}$ into a linguistic distribution assessment on $S^7$. By (\ref{mi}) and (\ref{mik}), the linguistic terms in $S^{13}$ can be transformed into the following 13 linguistic distribution assessments on $S^7$: $m_0^{7}=\{\langle s_0^7, 1\rangle\}$, $m_1^{7}=\{\langle s_0^7,0.5\rangle,\langle s_1^7,0.5\rangle\}$,
$m_2^{7}=\{\langle s_1^7,1\rangle\}$, $m_3^{7}=\{\langle s_1^7,0.5\rangle,\langle s_2^7,0.5\rangle\}$, $m_4^{7}=\{\langle s_2^7,1\rangle\}$, $m_5^{7}=\{\langle s_2^7,0.5\rangle,\langle s_3^7,0.5\rangle\}$, $m_6^{7}=\{\langle s_3^7,1\rangle\}$, $m_7^{7}=\{\langle s_3^7,0.5\rangle,\langle s_4^7,0.5\rangle\}$, $m_8^{7}=\{\langle s_4^7,1\rangle\}$, $m_9^{7}=\{\langle s_4^7,0.5\rangle,\langle s_5^7,0.5\rangle\}$, $m_{10}^{7}=\{\langle s_5^7,1\rangle\}$, $m_{11}^{7}=\{\langle s_5^7,0.5\rangle,\langle s_6^7,0.5\rangle\}$, $m_{12}^{7}=\{\langle s_6^7,1\rangle\}$. Moreover, the weight vector $\omega=(0,0,0,0.3,0,0,0.5,0,0,0.2,0,0,0)^{\rm T}$.

By (\ref{mi1}), it can be derived a linguistic distribution assessment on $S^7$ as $\{\langle s_1^7,0.15\rangle, \langle s_2^7,0.15\rangle,\langle s_3^7,0.5\rangle,\langle s_4^7,0.1\rangle,$\\$\langle s_5^7,0.1\rangle\}$.

In a similar manner, the linguistic distribution assessment $\{\langle s_1^7,0.25\rangle,\langle s_2^7,0.3\rangle,\langle s_3^7,0.45\rangle\}$ can be transformed into a linguistic distribution assessment on $S^5$ as $\{\langle s_0^5,0.0833\rangle, \langle s_1^5,0.3667\rangle,\langle s_2^5,0.55\rangle\}$.

Based on the aforementioned transformation procedures, we achieve the transformation of linguistic distribution assessments between any two linguistic scales. The remaining of this paper uses these procedures to solve large-scale MAGDM problems with multi-granular linguistic information.

\section{Large-scale linguistic MAGDM based on multi-granular linguistic distribution assessments}\label{sec:5}

In this section, an approach for linguistic large-scale MAGDM based on multi-granular linguistic distribution assessments is presented, which is suitable to deal with LGDM problems. The first novelty of the proposed approach is the use of linguistic distribution assessments to represent the assessments of the group, which keeps the maximum information elicited by decision makers of the group in initial stages of the decision process. Another novelty is that the proposed approach allows the use of multi-granular linguistic information, which provides a flexible way for decision makers with different background and knowledge to express their assessment information. First of all, the formulation of the linguistic large-scale MAGDM problem is introduced.

\subsection{Formulation of the large-scale linguistic MAGDM problem}
For the convenience of description, let $I=\{1,2,\ldots,n\}$, $J=\{1,2,\ldots,m\}$, $L=\{1,2,\ldots,q\}$ and $H=\{1,2,\ldots,r\}$. Consider the following linguistic large-scale MAGDM problem. Let $G=\{G_1,G_2,\ldots,G_n\}$ be a finite set of alternatives, $C=\{C_1,C_2,\ldots,C_m\}$ be the set of attributes, $D=\{d_1,d_2,\ldots,d_q\}$ be the set of a large group of decision makers. The weighting vector of the attributes is $w=(w_1,w_2,\ldots,w_m)^{\rm T}$, where $0 \leqslant w_j \leqslant 1,j\in J, \sum\limits_{j = 1}^m {{w_j}}  = 1$. In the decision making process, the decision makers provide their assessments for each alternative with respect to each attribute using linguistic terms.

To make it more convenient for decision makers to express their assessments over alternatives, multi-granular linguistic term sets are allowed in our MAGDM problem. Let $S=\{S^{g(1)},S^{g(2)},\ldots,S^{g(r)}\}$ be the linguistic term sets to be used by the decision makers, where $S^{g(h)}=\{s_0^{g(h)},s_1^{g(h)},\ldots, s_{g(h)-1}^{g(h)}\}$ is a linguistic term set with a granularity of $S^{g(h)}$, $h\in H$. During the decision process, each decision maker elicits his/her linguistic preferences in only one linguistic term set for his/her assessments. The more knowledge has the decision maker about the problem the more granularity. Conversely, the less knowledge the less granularity. Therefore, the set of decision makers can be divided into $r$ groups according to the linguistic term sets used. For convenience, let $D=\{D_1,D_2,\ldots,D_r\}$, where $D_h$ is the set of decision makers who select the linguistic term set $S^{g(h)}$, $h\in H$. Moreover, the assessment of the $i$th alternative with respect to the $j$th attribute provided by the $l$th decision maker is denoted by $x_{ij}^{l}$, then $x^l_{ij}\in S^{g(h)}$, $i\in I$, $j\in J$, $d_l\in D_h$, $h\in H$. The  decision makers assessments are summarized in Table \ref{tab:1}. The GDM problem must obtain a solution of the best alternative through fusing the information provided by the decision makers.

\begin{table}
\caption{Decision information of the decision makers}\label{tab:1}
\centering
\renewcommand{\arraystretch}{1.5}
\begin{tabular}{c|ccc|ccc|c|ccc}
\hline
\multirow{2}{*}{\centering Alternatives} &  \multicolumn{10}{c}{Attributes} \\\cline{2-11}
&\multicolumn{3}{c|}{$C_1$}& \multicolumn{3}{c|}{$C_2$}&{$\ldots$}&\multicolumn{3}{c}{$C_m$}\\\hline
$G_1$ & $x_{11}^1$ & $\ldots$ & $x_{11}^q$ & $x_{12}^1$ & $\ldots$ & $x_{12}^q$ & $\ldots$  & $x_{1m}^1$ & $\ldots$ & $x_{1m}^q$\\
$G_2$ & $x_{21}^1$ & $\ldots$ & $x_{21}^q$ & $x_{22}^1$ & $\ldots$ & $x_{22}^q$ & $\ldots$  & $x_{2m}^1$ & $\ldots$ & $x_{2m}^q$\\
$\ldots$ & \multicolumn{3}{c|}{$\ldots$} & \multicolumn{3}{c|}{$\ldots$} & $\ldots$ & \multicolumn{3}{c}{$\ldots$} \\
$G_n$ & $x_{n1}^1$ & $\ldots$ & $x_{n1}^q$ & $x_{n2}^1$ & $\ldots$ & $x_{n2}^q$ & $\ldots$  & $x_{nm}^1$ & $\ldots$ & $x_{nm}^q$\\
\hline
\end{tabular}
\end{table}

\subsection{The proposed MAGDM approach}

Here, it is proposed an approach for solving MAGDM problems dealing with multi-granular linguistic information modeled by linguistic distribution assessments. This decision process applies a multi-step aggregation method to the linguistic distribution assessments, by unifying them. Subsequently the weights of the attributes are determined for aggregating the attribute values and ranking the alternatives, eventually the collective assessments are represented in an easy understanding way. These steps are further detailed below.

\subsubsection{Representing decision makers' assessments by linguistic distribution assessments}\label{repre}

To keep the maximum information elicited by decision makers of the group in initial stages of the decision process, it is used linguistic distribution assessments to represent the linguistic information. As multi-granular linguistic distribution assessments are elicited by different decision makers, firstly collective assessments  over each alternative with respect to each attribute assessed in the same linguistic term set are computed.

Let $z_{ij}^h=\{\langle s_k^{g(h)},\beta_{ij,k}^h\rangle|k=0,1,\ldots,g(h)-1\}$ denote the linguistic distribution assessment on the $i$th alternative with respect to the $j$th attribute from the decision makers using the linguistic term set $S^{g(h)}$, $i\in I$, $j\in J$, $h\in H$, then the following two cases are considered.

\textbf{a) The decision makers are of equal importance.} In this case,
\bea\label{tran1}
\beta_{ij,k}^h=\dfrac{\#\{l|x_{ij}^{l}=s_k^{g(h)},l\in L\}}{\#{D_h}}, i\in I,\ j\in J,\ k=0,1,\ldots,g(h)-1, h\in H.
\eea
where $\#(\cdot)$ is the cardinality of the set.

Accordingly, the weight of the decision makers who utilize the linguistic term set $S^{g(h)}$ is obtained as $\omega_h={\#{D_h}}/{q}$, $h\in H$.

\textbf{b) The decision makers are of unequal importance.} Let $\lambda=(\lambda_1,\lambda_2,\ldots,\lambda_q)^{\rm T}$ be the weighting vector of the decision makers, where $0 \leqslant \lambda_l \leqslant 1,\ l\in  L,\ \sum\limits_{l = 1}^q {{\lambda_l}}  = 1$, then we have
\bea\label{tran2}
\beta_{ij,k}^h=\dfrac{\sum\nolimits_{l\in P_{ij}^k}\lambda_l}{\sum\nolimits_{d_l\in D_h}\lambda_l},
\eea
where $P_{ij,k}^h=\{l|x_{ij}^{l}=s_k^{g(h)}, l\in L\}$, $i\in I$, $j\in J$, $k=0,1,\ldots,g(h)-1$, $h\in H$.

Similarly, the weight of the decision makers who utilize the linguistic term set $S^{g(h)}$ is obtained as $\omega_h=\sum\nolimits_{d_l\in D_h}\lambda_l$, $h\in H$.

It is easy to verify that $z_{ij}^h$ derived by (\ref{tran1}) and (\ref{tran2}) are linguistic distribution assessments on $S^{g(h)}$, $i\in I$, $j\in J$, $h\in H$. A simple example to demonstrate this step is described below.
\begin{example}
Assume that five decision makers want to evaluate a new product by considering three attributes, including safety, cost and technical performance. The first two decision makers provide his linguistic assessments over the product using a linguistic term set $S^5=\{s_0^5,s_1^5,\ldots,s_4^5\}$ and the other three decision makers use a linguistic term set $S^7=\{s_0^7,s_1^7,\ldots,s_6^7\}$, as demonstrated in Table \ref{tab:2}.
\end{example}

\begin{table}[htbp]
\caption{Assessments of the new product}\label{tab:2}
\centering
\renewcommand{\arraystretch}{1.5}
\begin{tabular}{c|ccccc|ccccc|ccccc}
\hline
\multirow{2}{*}{\centering Alternatives} &  \multicolumn{15}{c}{Attributes} \\\cline{2-16}
&\multicolumn{5}{c|}{$C_1$: Safety}& \multicolumn{5}{c|}{$C_2$: Cost}&\multicolumn{5}{c}{$C_3$: Tech. P.}\\\hline
$G_1$ & $s_4^5$ & $s_3^5$ & $s_5^7$& $s_6^7$ & $s_6^7$ & $s_1^5$ & $s_2^5$ & $s_3^7$& $s_3^7$ & $s_4^7$  & $s_1^5$ & $s_1^5$ & $s_3^7$& $s_2^7$ & $s_2^7$\\
\hline
\end{tabular}
\end{table}

Let $S=\{S^5,S^7\}$ be the set of linguistic domains. If the three decision makers are of equal importance, then $\beta_{11,0}^1=\beta_{11,1}^1=\beta_{11,2}^1=0$, $\beta_{11,3}^1=\frac{\#\{2\}}{2}=0.5$, $\beta_{11,4}^1=\frac{\#\{1\}}{2}=0.5$. Hence, the collective assessment over $G_1$ with respect to $C_1$ from decision makers using $S^5$ can be denoted by a linguistic distribution assessment $z_{11}^1=\{\langle s_3^5,0.5\rangle,\langle s_4^5, 0.5\rangle\}$. In a similar manner, the collective assessment over $G_1$ with respect to $C_1$ from decision makers using $S^7$ is denoted by $z_{11}^2=\{\langle s_5^7,0.333\rangle,\langle s_6^7, 0.667\rangle\}$. The collective assessments over $G_1$ with respect to all the  attributes are showed in Table \ref{tab:3}.

\begin{table}[htbp]
\caption{Linguistic distribution assessments with equal importance}\label{tab:3}
\centering
\renewcommand{\arraystretch}{1.5}
\begin{tabular}{c|c|c|c}
\hline
\multirow{2}{*}{\centering $z_{ij}^h$} &  \multicolumn{3}{c}{Attributes} \\\cline{2-4}
&{$C_1$: Safety}& {$C_2$: Cost}&{$C_3$: Tech. P.}\\\hline
$z_{1j}^1$& $\{\langle s_3^5,0.5\rangle,\langle s_4^5, 0.5\rangle\}$& $\{\langle s_1^5,0.5\rangle,\langle s_2^5, 0.5\rangle\}$ & $\{\langle s_1^5,1\rangle\}$\\
$z_{1j}^2$&$\{\langle s_5^7,0.333\rangle,\langle s_6^7, 0.667\rangle\}$ & $\{\langle s_3^7,0.667\rangle,\langle s_4^7, 0.333\rangle\}$ & $\{\langle s_2^7,0.667\rangle,\langle s_3^7, 0.333\rangle\}$\\
\hline
\end{tabular}
\end{table}

If the weighting vector of the three decision makers is $\lambda=(0.2,0.3,0.2,0.15,0.15)^{\rm T}$, then $P_{11,0}^1=P_{11,1}^1=P_{11,2}^1=\phi$, $P_{11,3}^1=\{2\}$, $P_{11,4}^{1}=\{1\}$. It follows that $\beta_{11,0}^1=\beta_{11,1}^1=\beta_{11,2}^1=0$, $\beta_{11,3}^1=\lambda_2/(\lambda_1+\lambda_2)=0.6$, $\beta_{11,4}^1=\lambda_1/(\lambda_1+\lambda_2)=0.4$. Therefore, the the group's assessments over $G_1$ with respect to $C_1$  from decision makers using $S^5$ can be denoted by a linguistic distribution assessment $\{\langle s_3^5,0.4\rangle,\langle s_4^5, 0.6\rangle\}$. Similarly, the collective assessment over $G_1$ with respect to $C_1$ from decision makers using $S^7$ is denoted by $z_{11}^2=\{\langle s_5^7,0.4\rangle,\langle s_6^7, 0.6\rangle\}$. The collective assessments over $G_1$ with respect to all the attributes are showed in Table \ref{tab:4}.

\begin{table}[htbp]
\caption{Linguistic distribution assessments with unequal importance}\label{tab:4}
\centering
\renewcommand{\arraystretch}{1.5}
\begin{tabular}{c|c|c|c}
\hline
\multirow{2}{*}{\centering $z_{ij}^h$} &  \multicolumn{3}{c}{Attributes} \\\cline{2-4}
&{$C_1$: Safety}& {$C_2$: Cost}&{$C_3$: Tech. P.}\\\hline
$z_{1j}^1$& $\{\langle s_3^5,0.6\rangle,\langle s_4^5, 0.4\rangle\}$& $\{\langle s_1^5,0.4\rangle,\langle s_2^5, 0.6\rangle\}$ & $\{\langle s_1^5,1\rangle\}$\\
$z_{1j}^2$&$\{\langle s_5^7,0.4\rangle,\langle s_6^7, 0.6\rangle\}$ & $\{\langle s_3^7,0.7\rangle,\langle s_4^7, 0.3\rangle\}$ & $\{\langle s_2^7,0.6\rangle,\langle s_3^7, 0.4\rangle\}$\\
\hline
\end{tabular}
\end{table}

\subsubsection{Unifying the multi-granular linguistic distribution assessments}
Through a transformation process, the group's linguistic assessments over the alternatives with respect to each attribute can be denoted as $r$ decision matrices, whose elements are multi-granular linguistic distribution assessments, i.e.
\bea\label{matrix}
Z^{h}=(z_{ij}^{h})_{n\times m}=\left(
  \begin{array}{cccc}
 z_{11}^h &  z_{12}^h & \ldots  & z_{1m}^h \\
 z_{21}^h &  z_{21}^h & \ldots  & z_{2m}^1\\
 {\vdots} & {\vdots} & \vdots & {\vdots} \\
 z_{n1}^h &  z_{n2}^h & \ldots & z_{nm}^h\\
  \end{array}
\right), h\in H
\eea

To fuse these multi-granular linguistic distribution assessments and derive a collective opinion over each alternative, the procedures proposed in Section \ref{sec:4} are utilized. The granularity of the new linguistic term set $S^{g(h^*)}$ is calculated as
\bea\label{lcm1}
g(h^*)=LCM(g(1)-1,g(2)-1,\ldots,g(r)-1)+1.
\eea

By (\ref{multi2}), each $z_{ij}^h$ is transformed into a linguistic distribution assessment on $S^{g(h^*)}$ as $z_{ij}^{h'}=\{\langle s_k^{g(h^*)},\gamma_{ij,k}^h\rangle|k=0,1,\ldots,g(h^*)-1\}$, and
\ben\label{eqn:24}
\gamma_{ij,k}^h=\left\{\begin{array}{l}
\begin{split}
&\beta_{ij,l(h,k)}^h &{\rm if}\  l(h,k)&\in \{0,1,\ldots,g(h)-1\};\\
&0  &{\rm if}\  l(h,k)&\notin \{0,1,\ldots,g(h)-1\},\
\end{split}
\end{array}\right.
\een
where $l(h,k)=\dfrac{k*(g(h)-1)}{g(j^*)-1}$, $h\in H$, $k=0,1,\ldots,g(h^*)-1$.

\subsubsection{Aggregating the unified decision matrices}

Now, all the elements of the $r$ decision matrices are transformed into linguistic distribution assessments over the linguistic term set $S^{g(h^*)}$. By applying the DAWA operator, the collective assessments of the group on each alternative with respect to each attribute can be calculated, which are also linguistic distribution assessments.

 Let $z_{ij}=\{\langle s_k^{g(h^*)},\gamma_{ij,k}\rangle|k=0,1,\ldots,g(h^*)-1\}$ denote the collective assessment on the $i$th alternative with respect to the $j$th attribute, then
\bea
\gamma_{ij,k}=\sum\limits_{h=1}^{r} \omega_h \gamma_{ij,k}^h,\ i\in I,\ j\in J,\ k=0,1,\ldots,g(h^*)-1,
\eea
 where $\omega_h$ is the weight of the decision makers who select the linguistic term set $S^{g(h)}$, $h\in H$.

\subsubsection{Determining the weights of the attributes} Once the collective opinions with alternatives have been obtained, a collective decision matrix must be computed. It is then necessary to aggregate the attribute values to obtain the collective assessment of each alternative. Before the aggregation process, the weight of the attributes should be determined. In this paper, the maximum deviation approach \cite{Wang98see} is used to determine the weights in case they are not known as a priori.

The basic idea of the maximum deviation approach \cite{Wang98see,Wu07fss} consists of if an attribute makes the collective values among all the alternatives have obvious differences, then it plays an important role in choosing the best alternative. From the view of ranking alternatives, an attribute which has similar attribute values among alternatives should be assigned a small weight; otherwise, the attribute which has larger deviations among attribute values should be given a higher weight. Based on this idea, it is developed an approach to determine the weights of the attributes for decision making with linguistic distribution assessments.

By the DAWA operator, the collective assessment of each alternative can be denoted by $z_i=\{\langle s_k^{g(i^*)},\gamma_{i,k}\rangle|$\\$k=0,1,\ldots,g(h^*)-1\}$, where
\bea\label{aggregation}
\gamma_{i,k}=\sum\limits_{j=1}^{m}w_j\gamma_{ij,k},\ i\in I,\ k=0,1,\ldots,g(h^*)-1.
\eea

For the $j$th attribute, the deviation among all the alternatives is denoted by
\bea
\begin{aligned}
V_j(w)&=\sum\limits_{i=1}^{n}\sum\limits_{l=1}^{n}d(z_i,z_l)=\dfrac{1}{g(h^*)-1}\sum\limits_{i=1}^{n}\sum\limits_{l=1}^{n}\left|\sum\limits_{k=0}^{g(h^*)-1}k(\gamma_{i,k}-\gamma_{j,k})\right|\\
&=\dfrac{1}{g(h^*)-1}\sum\limits_{i=1}^{n}\sum\limits_{l=1}^{n}\left|\sum\limits_{k=0}^{g(h^*)-1}k(\gamma_{ij,k}-\gamma_{lj,k})\right|w_j,\ j\in J.
\end{aligned}
\eea

Therefore, the deviation among all the alternatives with respect to all the attributes is calculated
\bea
V(w)=\sum\limits_{j=1}^{m}V_j(w)=\dfrac{1}{g(h^*)-1}\sum\limits_{j=1}^{m}w_j\sum\limits_{i=1}^{n}\sum\limits_{l=1}^{n}\left|\sum\limits_{k=0}^{g(h^*)-1}k(\gamma_{ij,k}-\gamma_{lj,k})\right|.
\eea

Based on the maximum deviation approach, the following model is established to derive the weights of the attributes:
\ben\label{M-1}\tag{M-1}
\begin{split}
&\max\quad V(w)=\dfrac{1}{g(h^*)-1}\sum\limits_{j=1}^{m}w_j\sum\limits_{i=1}^{n}\sum\limits_{l=1}^{n}\left|\sum\limits_{k=0}^{g(h^*)-1}k(\gamma_{ij,k}-\gamma_{lj,k})\right|\\
&\begin{split}
{\text {s.t.}}\qquad &\sum\limits_{j=1}^m{w_j}^2=1\\
&w_j\geqslant 0, j\in J.
\end{split}
\end{split}
\een

By solving the model (\ref{M-1}) and normalizing the weighting vector, the weight of each attribute is derived by
\bea\label{weight}
w_j=\dfrac{\sum\limits_{i=1}^{n}\sum\limits_{l=1}^{n}\left|\sum\limits_{k=0}^{g(h^*)-1}k(\gamma_{ij,k}-\gamma_{lj,k})\right|}
{\sum\limits_{j=1}^m\sum\limits_{i=1}^{n}\sum\limits_{l=1}^{n}\left|\sum\limits_{k=0}^{g(h^*)-1}k(\gamma_{ij,k}-\gamma_{lj,k})\right|}
, j\in J.
\eea

If the weight information of attributes is partly known (please refer to the five cases in \cite{Xu08kbs,Zhang12kbs}), the following optimization model is established to derive the weights of the attributes:
\ben\label{M-2}\tag{M-2}
\begin{split}
&\max\quad V(w)=\dfrac{1}{g(h^*)-1}\sum\limits_{j=1}^{m}w_j\sum\limits_{i=1}^{n}\sum\limits_{l=1}^{n}\left|\sum\limits_{k=0}^{g(h^*)-1}k(\gamma_{ij,k}-\gamma_{lj,k})\right|\\
&\begin{split}
{\text {s.t.}}\qquad &\sum\limits_{j=1}^m{w_j}=1\\
& (w_1,w_2,\ldots,w_m) \in \Omega\\
&w_j\geqslant 0, j\in J,
\end{split}
\end{split}
\een
where $\Omega$ is the weighting vector space constructed by the partly known weight information.

By solving the model (\ref{M-2}), the weighting vector can also be obtained.

\subsubsection{Aggregating the attribute values and ranking the alternatives}
Once the weights of the attributes are determined, the collective assessment of each alternative can be calculated by aggregating the attribute values using (\ref{aggregation}). The collective assessments are denoted by $z_i=\{\langle s_k^{g(h^*)},\gamma_{i,k}\rangle|k=0,1,\ldots,g(h^*)-1\}$, $i\in I$.

For each $z_i$, the expectation values and the inaccuracy function values are calculated by Definitions \ref{def-4} and \ref{inacc} as
\bea
E(z_i)=\sum\limits_{k=0}^{g(h^*)-1}k\gamma_{i,k},\ i\in I
\eea
and
\bea
T(z_i)=-\sum\limits_{k=0}^{g(h^*)-1}\gamma_{i,k}\log_2 \gamma_{i,k},\ i\in I.
\eea

Based on the values of $E(z_i)$ and $T(z_i)$, the ranking of the alternatives can be derived by Definition \ref{def-6}. According to the ranking, the best alternative can be obtained.

\subsubsection{Representing the collective assessments}
By aggregating the attribute values, the collective assessment of each alternative is derived which is a linguistic distribution assessment on the linguistic term set $S^{g(i^*)}$. For such linguistic distribution assessments, it is hard for decision makers to understand the collective assessment of each alternative, since the linguistic distribution assessments are not defined on their initial linguistic term sets. To provide interpretable final linguistic results for decision makers, it is necessary to transform the derived collective assessments into linguistic distribution assessments using the initial linguistic term sets. Stage 2 in subsection \ref{sec:4-2} is utilized to achieve this goal.
By doing so, the decision makers can clearly know the overall assessments of the alternatives using their own linguistic term set as well as the proportion of each linguistic term.

To summarize, the procedures of the proposed MAGDM approach are given below, which it is also depicted in Fig. \ref{fig2}.

\begin{figure*}
  \centering
    \includegraphics[scale=0.75]{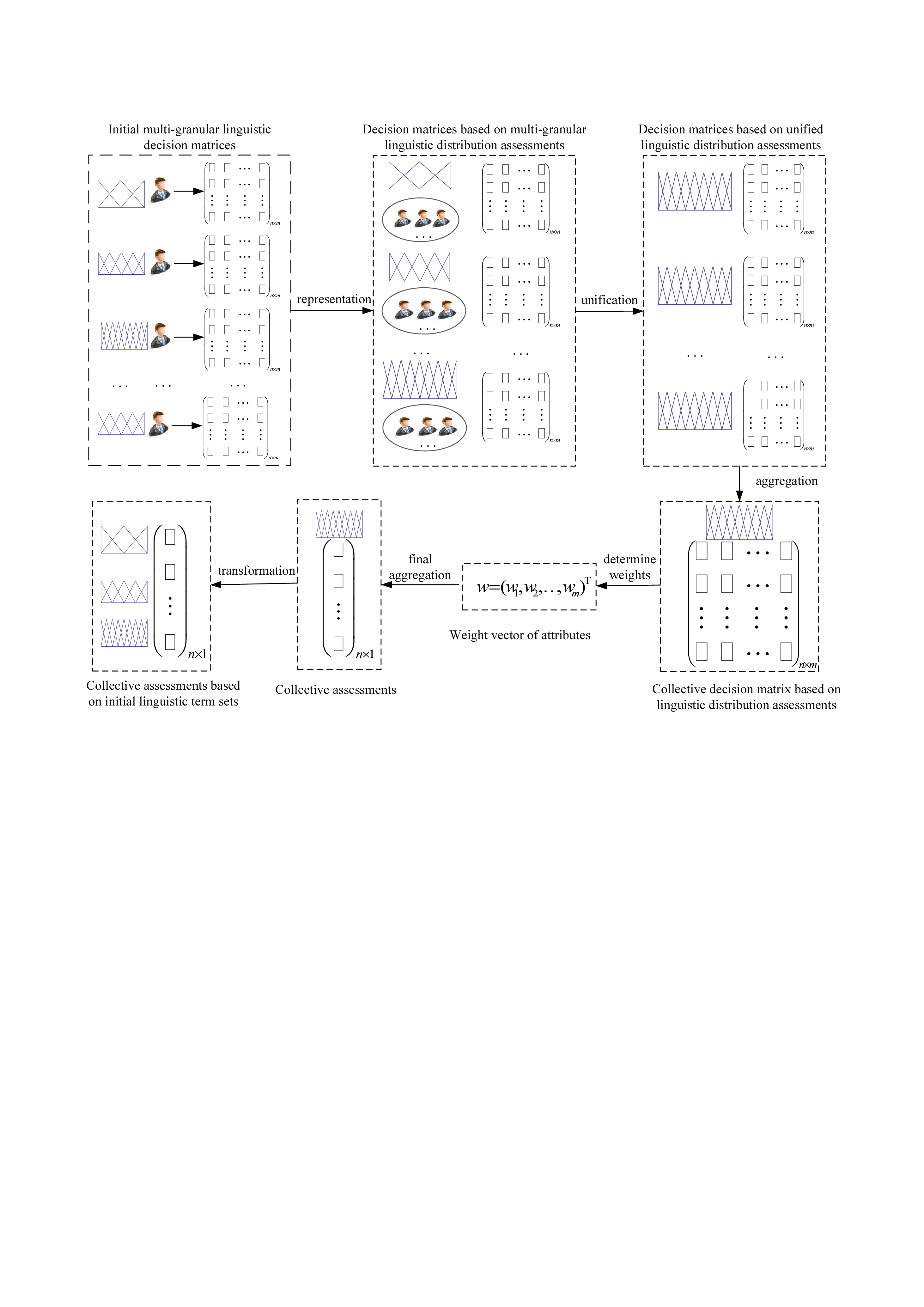}\\
      \caption{Flowchart of the decision making approach}\label{fig2}
\end{figure*}

\begin{enumerate}[\indent \bf Step 1:]
\item Gather the decision makers' assessments and represent the assessments from the decision makers who select the same linguistic term set using linguistic distribution assessments by (\ref{tran1}) or (\ref{tran2}). In this way, $r$ decision matrices are derived by (\ref{matrix}). Also, determine the weights of each decision matrix as $\omega=(\omega_1,\omega_2,\ldots,\omega_r)^{\rm T}$.
\item Calculate the granularity of the new linguistic term set $S^{g(h^*)}$ by (\ref{lcm1}) and use (\ref{eqn:24}) to unify the multi-granular linguistic distribution assessments.
\item Aggregate the $r$ unified decision matrices to derive the collective assessments on each alternative with respect to each attribute by (\ref{aggregation}).
\item If the weights of the attributes are completely known, go to Step 5; If the weights of the attributes are completely unknown, calculate the weights of the attributes by (\ref{weight}); If the weights of the attributes are partly known, solve the optimization model (\ref{M-2}) to derive the weights of attributes.
\item Aggregate the attribute values for  each alternative to derive the collective assessments by (\ref{aggregation}). Afterwards, calculate the expectation values and inaccuracy function values for each alternative and output the ranking of the alternatives by Definition \ref{def-6}.
\item If the decision makers want to know the collective assessment using the initial linguistic term sets, represent the collective assessments of each alternative using the initial linguistic term sets by Stage 2 in subsection \ref{sec:4-2}.
\end{enumerate}

\section{An illustrative example}\label{sec:6}

In this section, an example for talent recruitment is used to demonstrate the proposed MAGDM approach. A university in China was intended to recruit a dean for the School of Business. The recruitment process was as follows. First, the university released an opening recruitment announcement on the website. Any people who satisfied the basic recruitment conditions could apply for the position using the online application system before the deadline. After receiving applications from candidates at home and abroad, the staffs of the personnel department made a strict selection by checking the application documents. Finally, four candidates ($G_1$, $G_2$, $G_3$, $G_4$) entered the interview for further selection. To make the final selection as fair as possible, a committee which was composed of 24 members from the academic board of the university was established. After making face-to-face interviews with the four candidates, each committee member was asked to provide their assessments over the four candidates with respect to the following four criteria: $C_1$: Academic background and influence; $C_2$: Leadership; $C_3$: Research and teaching experiences; $C_4$: International exchange and cooperation.

To facilitate the evaluation process, the committee members were allowed to provide their assessments using multi-granular linguistic term sets, i.e. each committee member could select a linguistic term set for assessment according to his/her interest. The linguistic term sets used were the following ones: $S^{g(1)}=S^5=\{s_0^5: \text{poor}, s_1^5:\text{slightly poor}, s_2^5:\text{fair},s_3^5:\text{slightly good},s_4^5: \text{good}\}$; $S^{g(2)}=S^7=\{s_0^7: \text{very poor}, s_1^7:\text{poor}, s_2^7:\text{slightly poor},s_3^7:\text{fair},s_4^5: \text{slightly good},s_5^7:\text{good},s_6^7:\text{very good}\}$; $S^{g(3)}=S^9=\{s_0^9: \text{extremely poor},s_1^9:\text{very poor}, s_2^9:\text{poor}, s_3^9:\text{slightly poor},s_4^9:\text{fair},s_5^9: \text{slightly good},s_6^9:\text{good},s_7^9:\text{very good},s_8^9: \text{extremely good}\}$.

Finally, the number of the committee members who selected $S^5$, $S^7$ and $S^9$ for assessment were 10, 8 and 6, respectively. The proposed MAGDM approach is then employed to select the most appropriate candidate.

\textbf{Step 1:} After gathering the committee members' multi-granular linguistic assessment, we conduct the initial fusion of information and represent the collective assessment over each candidate with respect to each criteria from the committee members who select the same linguistic term set using multi-granular linguistic distribution assessments (see the procedures in subsection \ref{repre}).
The results are showed in Tables \ref{tab:5}-\ref{tab:7}.

\begin{table}[htbp]\small
\caption{Linguistic distribution assessments using $S^5$}\label{tab:5}
\centering
\renewcommand{\arraystretch}{1.5}
\begin{tabular}{lp{3.6cm}p{3.6cm}p{3.6cm}p{3.6cm}}
  \toprule
$O_1$ & $C_1$ & $C_2$ & $C_3$ & $C_4$   \\\midrule
  $G_1$&  $\{\langle s_3^5,0.4\rangle,\langle s_4^5,0.6\rangle  \}$&$\{\langle s_2^5,0.2\rangle,\langle s_3^5,0.8\rangle  \}$&$\{\langle s_3^5,0.8\rangle,\langle s_4^5,0.2\rangle \}$&$\{\langle s_3^5,1\rangle  \}$\\
  $G_2$&  $\{\langle s_2^5,0.4\rangle,\langle s_3^5,0.2\rangle$, $\langle s_4^5,0.4\rangle \}$&$\{\langle s_3^5,0.8\rangle,\langle s_4^5,0.2\rangle  \}$&$\{\langle s_2^5,0.2\rangle,\langle s_3^5,0.4\rangle$, $\langle s_4^5,0.4\rangle \}$&$\{\langle s_2^5,0.6\rangle, \langle s_3^5,0.4\rangle \}$\\
  $G_3$&  $\{\langle s_1^5,0.2\rangle,\langle s_2^5,0.4\rangle$, $\langle s_3^5,0.4\rangle  \}$&$\{\langle s_2^5,0.6\rangle,\langle s_3^5,0.4\rangle  \}$&$\{\langle s_4^5,1\rangle  \}$& $\{\langle s_1^5,0.4\rangle,\langle s_2^5,0.4\rangle $, $\langle s_3^5,0.2\rangle\}$\\
  $G_4$&  $\{\langle s_4^5,1\rangle \}$&$\{\langle s_3^5,0.6\rangle,\langle s_4^5,0.4\rangle  \}$&$\{\langle s_1^5,0.4\rangle,\langle s_2^5,0.4\rangle$, $\langle s_3^5,0.2\rangle \}$&$\{\langle s_3^5,0.2\rangle,\langle s_4^5,0.8\rangle  \}$\\
 \bottomrule
\end{tabular}
\end{table}

\begin{table}[htbp]\small
\caption{Linguistic distribution assessments using $S^7$}\label{tab:6}
\centering
\renewcommand{\arraystretch}{1.5}
\begin{tabular}{lp{3.6cm}p{3.6cm}p{3.6cm}p{3.6cm}}
  \toprule
$O_2$ & $C_1$ & $C_2$ & $C_3$ & $C_4$   \\\midrule
  $G_1$&  $\{\langle s_3^7,0.25\rangle,\langle s_4^7,0.5\rangle$, $\langle s_5^7,0.25\rangle \}$&$\{\langle s_1^7,0.25\rangle,\langle s_2^7,0.25\rangle$, $\langle s_4^7,0.5\rangle   \}$&$\{\langle s_4^7,0.5\rangle,\langle s_6^7,0.5\rangle \}$&$\{\langle s_3^7,0.25\rangle ,\langle s_4^7,0.75\rangle  \}$\\
  $G_2$&  $\{\langle s_3^7,0.5\rangle ,\langle s_4^7,0.5\rangle \}$&$\{\langle s_3^7,0.25\rangle,\langle s_4^7,0.25\rangle$, $\langle s_5^7,0.5\rangle  \}$&$\{\langle s_3^7,0.25\rangle,\langle s_4^7,0.5\rangle$, $\langle s_5^7,0.25\rangle \}$&$\{\langle s_5^7,0.5\rangle, \langle s_6^7,0.5\rangle \}$\\
  $G_3$&  $\{\langle s_3^7,0.5\rangle,\langle s_4^7,0.5\rangle\}$&$\{\langle s_3^7,0.25\rangle,\langle s_4^7,0.75\rangle  \}$&$\{\langle s_4^7,0.25\rangle ,\langle s_5^7,0.25\rangle$, $\langle s_6^7,0.5\rangle \}$& $\{\langle s_0^7,0.25\rangle,\langle s_2^7,0.75\rangle\}$\\
  $G_4$&  $\{\langle s_4^7,0.5\rangle ,\langle s_5^7,0.25\rangle$, $\langle s_6^7,0.25\rangle \}$&$\{\langle s_5^7,0.5\rangle,\langle s_6^7,0.5\rangle  \}$&$\{\langle s_2^7,0.25\rangle,\langle s_3^7,0.25\rangle$, $\langle s_4^7,0.5\rangle \}$&$\{\langle s_5^7,0.5\rangle,\langle s_6^7,0.5\rangle  \}$\\
 \bottomrule
\end{tabular}
\end{table}

\begin{table}[htbp]\small
\caption{Linguistic distribution assessments using $S^9$}\label{tab:7}
\centering
\renewcommand{\arraystretch}{1.5}\small
\begin{tabular}{lp{3.6cm}p{3.6cm}p{3.6cm}p{3.6cm}}
  \toprule
$O_3$ & $C_1$ & $C_2$ & $C_3$ & $C_4$   \\\midrule
  $G_1$&  $\{\langle s_5^9,0.333\rangle,\langle s_6^9,0.167\rangle$, $\langle s_7^9,0.5\rangle \}$&$\{\langle s_4^9,0.167\rangle,\langle s_5^9, 0.833\rangle\}$&$\{\langle s_6^9,0.5\rangle,\langle s_7^9,0.167\rangle$, $\langle s_8^9,0.333\rangle  \}$&$\{\langle s_5^9,0.5\rangle ,\langle s_6^9,0.167\rangle $, $\langle s_7^9,0.333\rangle  \}$\\
  $G_2$&  $\{\langle s_3^9,0.333\rangle ,\langle s_5^9,0.167\rangle$, $\langle s_6^9,0.5\rangle  \}$&$\{\langle s_6^9,0.5\rangle,\langle s_7^9,0.5\rangle\}$&$\{\langle s_4^9,0.333\rangle,\langle s_6^9,0.5\rangle $, $\langle s_7^9,0.167\rangle \}$&$\{\langle s_2^9,0.333\rangle, \langle s_4^9,0.167\rangle$, $\langle s_5^9,0.5\rangle \}$\\
  $G_3$&  $\{\langle s_4^9,0.333\rangle,\langle s_5^9,0.5\rangle$, $\langle s_6^9,0.167\rangle\}$&$\{\langle s_6^9,0.333\rangle,\langle s_7^9,0.667\rangle  \}$&$\{\langle s_7^9,0.5\rangle ,\langle s_8^9,0.5\rangle \}$& $\{\langle s_2^9,0.333\rangle,\langle s_3^9,0.333\rangle$, $\langle s_4^9,0.333\rangle\}$\\
  $G_4$&  $\{\langle s_6^9,0.333\rangle ,\langle s_7^9,0.667\rangle\}$&$\{\langle s_7^9,0.333\rangle, \langle s_8^9,0.667\rangle  \}$&$\{\langle s_3^9,0.333\rangle,\langle s_4^9,0.5\rangle$, $\langle s_5^9,0.167\rangle \}$&$\{\langle s_5^9,0.5\rangle,\langle s_6^9,0.167\rangle$, $\langle s_8^9,0.333\rangle   \}$\\
 \bottomrule
\end{tabular}
\end{table}

As the committee members were of equal importance, the weights of the committee members who selected $S^5$, $S^7$ and $S^9$ for assessment were 5/12, 1/3 and 1/4, respectively, i.e. $\omega=(5/12, 1/3 , 1/4)^{\rm T}$.

\textbf{Step 2:} The granularity of the new linguistic term set $g(h^*)$ is calculated. Since $g(1)=5$, $g(2)=7$, $g(3)=9$, we have $g(h^*)=LCM(4,6,8)+1=25$. By (\ref{eqn:24}), the linguistic distribution assessments in Tables \ref{tab:5} - \ref{tab:7} are transformed into linguistic distribution assessments on $S^{g(h^*)}$, as showed in Tables \ref{tab:8} - \ref{tab:10}.

\begin{table}[htbp]
\caption{Individual decision matrix $Z^{1'}=(z_{ij}^{1'})_{4\times 4}$}\label{tab:8}
\centering
\renewcommand{\arraystretch}{1.5}\small
\begin{tabular}{lp{3.6cm}p{3.6cm}p{3.6cm}p{3.6cm}}
  \toprule
$O_1$ & $C_1$ & $C_2$ & $C_3$ & $C_4$   \\\midrule
  $G_1$&  $\{\langle s_{18}^{25},0.4\rangle,\langle s_{24}^{25},0.6\rangle  \}$&$\{\langle s_{12}^{25},0.2\rangle,\langle s_{18}^{25},0.8\rangle  \}$&$\{\langle s_{18}^{25},0.8\rangle,\langle s_{24}^{25},0.2\rangle \}$&$\{\langle s_{18}^{25},1\rangle  \}$\\
  $G_2$&  $\{\langle s_{12}^{25},0.4\rangle,\langle s_{18}^{25},0.2\rangle $, $\langle s_{24}^{25},0.4\rangle \}$&$\{\langle s_{18}^{25},0.8\rangle,\langle s_{24}^{25},0.2\rangle  \}$&$\{\langle s_{12}^{25},0.2\rangle,\langle s_{18}^{25},0.4\rangle$, $\langle s_{24}^{25},0.4\rangle \}$&$\{\langle s_{12}^{25},0.6\rangle, \langle s_{18}^{25},0.4\rangle \}$\\
  $G_3$&  $\{\langle s_6^{25},0.2\rangle,\langle s_{12}^{25},0.4\rangle$, $\langle s_{18}^{25},0.4\rangle  \}$&$\{\langle s_{12}^{25},0.6\rangle,\langle s_{18}^{25},0.4\rangle  \}$&$\{\langle s_{24}^{25},1\rangle  \}$& $\{\langle s_{6}^{25},0.4\rangle,\langle s_{12}^{25},0.4\rangle$, $\langle s_{18}^{25},0.2\rangle\}$\\
  $G_4$&  $\{\langle s_{24}^{25},1\rangle \}$&$\{\langle s_{18}^{25},0.6\rangle,\langle s_{24}^{25},0.4\rangle  \}$&$\{\langle s_6^{25},0.4\rangle,\langle s_{12}^{25},0.4\rangle$, $\langle s_{18}^{25},0.2\rangle \}$&$\{\langle s_{18}^{25},0.2\rangle,\langle s_{24}^{25},0.8\rangle  \}$\\
 \bottomrule
\end{tabular}
\end{table}

\begin{table}[htbp]
\caption{Individual decision matrix $Z^{2'}=(z_{ij}^{2'})_{4\times 4}$}\label{tab:9}
\centering
\renewcommand{\arraystretch}{1.5}\small
\begin{tabular}{lp{3.6cm}p{3.6cm}p{3.6cm}p{3.6cm}}
  \toprule
$O_2$ & $C_1$ & $C_2$ & $C_3$ & $C_4$   \\\midrule
  $G_1$&  $\{\langle s_{12}^{25},0.25\rangle,\langle s_{16}^{25},0.5\rangle$, $\langle s_{20}^{25},0.25\rangle \}$&$\{\langle s_{4}^{25},0.25\rangle,\langle s_{8}^{25},0.25\rangle$, $\langle s_{16}^{25},0.5\rangle   \}$&$\{\langle s_{16}^{25},0.5\rangle,\langle s_{24}^{25},0.5\rangle \}$&$\{\langle s_{12}^{25},0.25\rangle ,\langle s_{16}^{25},0.75\rangle  \}$\\
  $G_2$&  $\{\langle s_{12}^{25},0.5\rangle ,\langle s_{16}^{25},0.5\rangle \}$&$\{\langle s_{12}^{25},0.25\rangle,\langle s_{16}^{25},0.25\rangle$, $\langle s_{20}^{25},0.5\rangle  \}$&$\{\langle s_{12}^{25},0.25\rangle,\langle s_{16}^{25},0.5\rangle$, $\langle s_{20}^{25},0.25\rangle \}$&$\{\langle s_{20}^{25},0.5\rangle, \langle s_{24}^{25},0.5\rangle \}$\\
  $G_3$&  $\{\langle s_{12}^{25},0.5\rangle,\langle s_{16}^{25},0.5\rangle\}$&$\{\langle s_{12}^{25},0.25\rangle,\langle s_{16}^{25},0.75\rangle  \}$&$\{\langle s_{16}^{25},0.25\rangle ,\langle s_{20}^{25},0.25\rangle$, $\langle s_{24}^{25},0.5\rangle \}$& $\{\langle s_{0}^{25},0.25\rangle,\langle s_{8}^{25},0.75\rangle\}$\\
  $G_4$&  $\{\langle s_{16}^{25},0.5\rangle ,\langle s_{20}^{25},0.25\rangle$, $\langle s_{24}^{25},0.25\rangle \}$&$\{\langle s_{20}^{25},0.5\rangle,\langle s_{24}^{25},0.5\rangle  \}$&$\{\langle s_{8}^{25},0.25\rangle,\langle s_{12}^{25},0.25\rangle$, $\langle s_{16}^{25},0.5\rangle \}$&$\{\langle s_{20}^{25},0.5\rangle,\langle s_{24}^{25},0.5\rangle  \}$\\
 \bottomrule
\end{tabular}
\end{table}

\begin{table}[htbp]\small
\caption{Individual decision matrix $Z^{3'}=(z_{ij}^{3'})_{4\times 4}$}\label{tab:10}
\centering
\renewcommand{\arraystretch}{1.5}
\begin{tabular}{lp{3.6cm}p{3.6cm}p{3.6cm}p{3.6cm}}
  \toprule
$O_3$ & $C_1$ & $C_2$ & $C_3$ & $C_4$   \\\midrule
  $G_1$&  $\{\langle s_{15}^{25},0.333\rangle,\langle s_{18}^{25},0.167\rangle$, $\langle s_{21}^{25},0.5\rangle \}$&$\{\langle s_{12}^{25},0.167\rangle,\langle s_{15}^{25}, 0.833\rangle\}$&$\{\langle s_{18}^{25},0.5\rangle,\langle s_{21}^{25},0.167\rangle$, $\langle s_{24}^{25},0.333\rangle  \}$&$\{\langle s_{15}^{25},0.5\rangle ,\langle s_{18}^{25},0.167\rangle $, $\langle s_{21}^{25},0.333\rangle  \}$\\
  $G_2$&  $\{\langle s_{9}^{25},0.333\rangle ,\langle s_{15}^{25},0.167\rangle$, $\langle s_{18}^{25},0.5\rangle  \}$&$\{\langle s_{18}^{25},0.5\rangle,\langle s_{21}^{25},0.5\rangle\}$&$\{\langle s_{12}^{25},0.333\rangle,\langle s_{18}^{25},0.5\rangle $, $\langle s_{21}^{25},0.167\rangle \}$&$\{\langle s_{6}^{25},0.333\rangle, \langle s_{12}^{25},0.167\rangle$, $\langle s_{15}^{25},0.5\rangle \}$\\
  $G_3$&  $\{\langle s_{12}^{25},0.333\rangle,\langle s_{15}^{25},0.5\rangle$, $\langle s_{18}^{25},0.167\rangle\}$&$\{\langle s_{18}^{25},0.333\rangle,\langle s_{21}^{25},0.667\rangle  \}$&$\{\langle s_{21}^{25},0.5\rangle ,\langle s_{24}^{25},0.5\rangle \}$& $\{\langle s_{6}^{25},0.333\rangle,\langle s_{9}^{25},0.333\rangle$, $\langle s_{12}^{25},0.333\rangle\}$\\
  $G_4$&  $\{\langle s_{18}^{25},0.333\rangle ,\langle s_{21}^{25},0.667\rangle\}$&$\{\langle s_{21}^{25},0.333\rangle, \langle s_{24}^{25},0.667\rangle  \}$&$\{\langle s_{9}^{25},0.333\rangle,\langle s_{12}^{25},0.5\rangle$, $\langle s_{15}^{25},0.167\rangle \}$&$\{\langle s_{15}^{25},0.5\rangle,\langle s_{18}^{25},0.167\rangle$, $\langle s_{24}^{25},0.333\rangle   \}$\\
 \bottomrule
\end{tabular}
\end{table}

\begin{table}[htbp]\small
\caption{Collective decision matrix $Z=(z_{ij})_{4\times 4}$}\label{tab:11}
\centering\small
\renewcommand{\arraystretch}{1.5}
\begin{tabular}{lp{3.6cm}p{3.6cm}p{3.6cm}p{3.6cm}}
  \toprule
 & $C_1$ & $C_2$ & $C_3$ & $C_4$   \\\midrule
  $G_1$&  $\{\langle s_{12}^{25},0.083\rangle,\langle s_{15}^{25},0.083\rangle$, $\langle s_{16}^{25},0.167\rangle,\langle s_{18}^{25},0.209\rangle$, $\langle s_{20}^{25},0.083\rangle,\langle s_{21}^{25},0.125\rangle$, $\langle s_{24}^{25},0.25\rangle \}$ & $\{\langle s_{4}^{25},0.083\rangle,\langle s_{8}^{25},0.083\rangle$, $\langle s_{12}^{25},0.125\rangle,\langle s_{15}^{25},0.209\rangle$, $\langle s_{16}^{25},0.167\rangle,\langle s_{18}^{25},0.333\rangle$ & $\{\langle s_{16}^{25},0.167\rangle,\langle s_{18}^{25},0.458\rangle$, $\langle s_{21}^{25},0.042\rangle,\langle s_{24}^{25},0.333\rangle \}$ & $\{\langle s_{12}^{25},0.083\rangle,\langle s_{15}^{25},0.125\rangle$, $\langle s_{16}^{25},0.25\rangle,\langle s_{18}^{25},0.459\rangle$, $\langle s_{21}^{25},0.083\rangle \}$\\
  $G_2$&  $\{\langle s_{9}^{25},0.083\rangle,\langle s_{12}^{25},0.333\rangle$, $\langle s_{15}^{25},0.042\rangle,\langle s_{16}^{25},0.167\rangle$, $\langle s_{18}^{25},0.208\rangle,\langle s_{24}^{25},0.167\rangle$ & $\{\langle s_{12}^{25},0.083\rangle,\langle s_{16}^{25},0.083\rangle$, $\langle s_{18}^{25},0.459\rangle,\langle s_{20}^{25},0.167\rangle$, $\langle s_{21}^{25},0.125\rangle, \langle s_{24}^{25},0.083\rangle \}$ & $\{\langle s_{12}^{25},0.25\rangle,\langle s_{16}^{25},0.167\rangle$, $\langle s_{18}^{25},0.291\rangle,\langle s_{20}^{25},0.083\rangle$, $\langle s_{21}^{25},0.042\rangle, \langle s_{24}^{25},0.167\rangle \}$ & $\{\langle s_{6}^{25},0.083\rangle,\langle s_{12}^{25},0.291\rangle$, $\langle s_{15}^{25},0.125\rangle,\langle s_{18}^{25},0.167\rangle$, $\langle s_{20}^{25},0.167\rangle, \langle s_{24}^{25},0.167\rangle \}$\\
  $G_3$&  $\{\langle s_{6}^{25},0.083\rangle,\langle s_{12}^{25},0.417\rangle$, $\langle s_{15}^{25},0.125\rangle,\langle s_{16}^{25},0.167\rangle$, $\langle s_{18}^{25},0.208\rangle \}$ & $\{\langle s_{12}^{25},0.333\rangle,\langle s_{16}^{25},0.25\rangle$, $\langle s_{18}^{25},0.25\rangle,\langle s_{21}^{25},0.167\rangle \}$ & $\{\langle s_{16}^{25},0.083\rangle,\langle s_{20}^{25},0.083\rangle$, $\langle s_{21}^{25},0.125\rangle,\langle s_{24}^{25},0.709\rangle \}$
   & $\{\langle s_{0}^{25},0.083\rangle,\langle s_{6}^{25},0.25\rangle$, $\langle s_{8}^{25},0.25\rangle,\langle s_{9}^{25},0.083\rangle$, $\langle s_{12}^{25},0.25\rangle, \langle s_{18}^{25},0.084\rangle \}$\\
  $G_4$&  $\{\langle s_{16}^{25},0.167\rangle,\langle s_{18}^{25},0.083\rangle$, $\langle s_{20}^{25},0.083\rangle,\langle s_{21}^{25},0.167\rangle$, $\langle s_{24}^{25},0.5\rangle \}$ & $\{\langle s_{18}^{25},0.25\rangle,\langle s_{20}^{25},0.167\rangle$, $\langle s_{21}^{25},0.083\rangle,\langle s_{24}^{25},0.5\rangle \}$ & $\{\langle s_{6}^{25},0.167\rangle,\langle s_{8}^{25},0.083\rangle$, $\langle s_{9}^{25},0.083\rangle,\langle s_{12}^{25},0.375\rangle$, $\langle s_{15}^{25},0.042\rangle,\langle s_{16}^{25},0.167\rangle$, $\langle s_{18}^{25},0.083\rangle \}$& $\{\langle s_{15}^{25},0.125\rangle,\langle s_{18}^{25},0.125\rangle$, $\langle s_{20}^{25},0.167\rangle,\langle s_{24}^{25},0.583\rangle \}$\\
 \bottomrule
\end{tabular}
\end{table}

\textbf{Step 3:} The three decision matrices are aggregated using the DAWA operator with a weighting vector $\omega=(5/12, 1/3 , 1/4)^{\rm T}$. The aggregated decision matrix is demonstrated in Table \ref{tab:11}.

\textbf{Step 4:} As the weights of the criteria are unknown, we use (\ref{weight}) to determine the weights of the criteria. The weighting vector is derived as $w=(0.2079,0.1968,0.2827, 0.3126)^{\rm T}$.

\textbf{Step 5:} By applying the DAWA operator,  the collective assessments of the four candidates ($z_i$) are calculated and showed in Table \ref{tab:12}.

\begin{table}[htbp]\small
\caption{Collective assessments of the four candidates}\label{tab:12}
\centering
\renewcommand{\arraystretch}{1.5}\small
\begin{tabular}{p{1cm}p{15.1cm}}
  \toprule
$G_i$  & Collective assessments \\\midrule
  $G_1$&  $\{\langle s_{4}^{25},0.016\rangle,\langle s_{8}^{25},0.016\rangle, \langle s_{12}^{25},0.068\rangle,\langle s_{15}^{25},0.098\rangle, \langle s_{16}^{25},0.193\rangle,\langle s_{18}^{25},0.382\rangle, \langle s_{20}^{25},0.017\rangle, \langle s_{21}^{25},0.064\rangle$, $\langle s_{24}^{25},0.146\rangle\}$\\
  $G_2$&  $\{\langle s_{6}^{25},0.026\rangle,\langle s_{9}^{25},0.017\rangle, \langle s_{12}^{25},0.248\rangle,\langle s_{15}^{25},0.048\rangle, \langle s_{16}^{25},0.098\rangle,\langle s_{18}^{25},0.268\rangle, \langle s_{20}^{25},0.109\rangle, \langle s_{21}^{25},0.036\rangle$, $\langle s_{24}^{25},0.150\rangle \}$\\
  $G_3$&  $\{\langle s_{0}^{25},0.026\rangle,\langle s_{6}^{25},0.096\rangle, \langle s_{8}^{25},0.078\rangle,\langle s_{9}^{25},0.026\rangle, \langle s_{12}^{25},0.230\rangle,\langle s_{15}^{25},0.026\rangle, \langle s_{16}^{25},0.107\rangle, \langle s_{18}^{25},0.119\rangle$, $\langle s_{20}^{25},0.024\rangle, \langle s_{21}^{25},0.068\rangle, \langle s_{24}^{25},0.200\rangle \}$\\
  $G_4$&  $\{\langle s_{6}^{25},0.047\rangle,\langle s_{8}^{25},0.024\rangle, \langle s_{9}^{25},0.023\rangle,\langle s_{12}^{25},0.106\rangle, \langle s_{15}^{25},0.051\rangle,\langle s_{16}^{25},0.082\rangle, \langle s_{18}^{25},0.129\rangle, \langle s_{20}^{25},0.102\rangle$, $\langle s_{21}^{25},0.051\rangle, \langle s_{24}^{25},0.385\rangle \}$\\
 \bottomrule
\end{tabular}
\end{table}

Afterwards, it is computed the expectation values of the four candidates' collective assessments. For each candidate, it is obtained
$E(z_1)=(s_{18}^{25},-0.38)$, $E(z_2)=(s_{17}^{25},-0.07)$, $E(z_3)=(s_{15}^{25},0.15)$, $E(z_4)=(s_{18}^{25},0.70)$, which results in a ranking $G_{4}\succ G_{1}\succ G_{2}\succ G_{3}$. As a result, the best candidate is $G_{4}$.

\textbf{Step 6:} The committee members want to know the collective assessments of each candidate, so the procedures of Stage 2 in subsection IV-B are used to transform each $z_i$ into linguistic distribution assessments of the initial linguistic term sets. The results are demonstrated in Tables \ref{tab:13}-\ref{tab:15}. From Tables \ref{tab:13} - \ref{tab:15}, we can obverse the collective assessments of the candidates. For instance, from Table \ref{tab:13} we can find that the collective assessment of $G_2$ is mainly about $s_2^5$ and $s_3^5$, while $G_4$ is about $s_3^5$ and $s_4^5$. Besides, we can also obtain the proportion distribution of the linguistic terms, which reflects the tendencies of the assessments. Therefore, the use of linguistic distribution assessments can provide more information about the assessments over alternatives.

\begin{table}[htbp]\small
\caption{Collective assessments of the four candidates using linguistic term set $S^5$}\label{tab:13}
\centering
\renewcommand{\arraystretch}{1.5}\small
\begin{tabular}{p{1.5cm}p{10cm}}
  \toprule
$G_i$  & Collective assessments  \\\midrule
  $G_1$&  $\{\langle s_{0}^{5},0.006\rangle,\langle s_{1}^{5},0.022\rangle, \langle s_{2}^{5},0.186\rangle,\langle s_{3}^{5},0.602\rangle, \langle s_{4}^{5},0.184\rangle\}$\\
  $G_2$&  $\{\langle s_{1}^{5},0.035\rangle, \langle s_{2}^{5},0.313\rangle,\langle s_{3}^{5},0.448\rangle, \langle s_{4}^{5},0.204\rangle\}$\\
  $G_3$&  $\{\langle s_{0}^{5},0.026\rangle,\langle s_{1}^{5},0.161\rangle, \langle s_{2}^{5},0.318\rangle,\langle s_{3}^{5},0.253\rangle, \langle s_{4}^{5},0.242\rangle\}$\\
  $G_4$&  $\{\langle s_{1}^{5},0.075\rangle, \langle s_{2}^{5},0.178\rangle,\langle s_{3}^{5},0.303\rangle, \langle s_{4}^{5},0.444\rangle\}$\\
 \bottomrule
\end{tabular}
\end{table}

\begin{table}[htbp]\small
\caption{Collective assessments of the four candidates using linguistic term set $S^7$}\label{tab:14}
\centering
\renewcommand{\arraystretch}{1.5}\small
\begin{tabular}{p{1cm}p{13cm}}
  \toprule
$G_i$  & Collective assessments  \\\midrule
  $G_1$&  $\{\langle s_{1}^{7},0.016\rangle,\langle s_{2}^{7},0.017\rangle, \langle s_{3}^{7},0.092\rangle,\langle s_{4}^{7},0.457\rangle, \langle s_{5}^{7},0.256\rangle,\langle s_{6}^{7},0.162\rangle\}$\\
  $G_2$&  $\{\langle s_{1}^{7},0.013\rangle,\{\langle s_{2}^{7},0.026\rangle,\langle s_{3}^{7},0.264\rangle, \langle s_{4}^{7},0.268\rangle,\langle s_{5}^{7},0.270\rangle, \langle s_{6}^{7},0.159\rangle\}$\\
  $G_3$&  $\{\langle s_{0}^{7},0.026\rangle,\{\langle s_{1}^{7},0.048\rangle,\langle s_{2}^{7},0.145\rangle, \langle s_{3}^{7},0.244\rangle,\langle s_{4}^{7},0.186\rangle, \langle s_{5}^{7},0.134\rangle,\langle s_{6}^{7},0.217\rangle\}$\\
  $G_4$&  $\{\langle s_{1}^{7},0.024\rangle,\langle s_{2}^{7},0.065\rangle, \langle s_{3}^{7},0.125\rangle,\langle s_{4}^{7},0.184\rangle, \langle s_{5}^{7},0.205\rangle,\langle s_{6}^{7},0.397\rangle\}$\\
 \bottomrule
\end{tabular}
\end{table}

\begin{table}[htbp]\small
\caption{Collective assessments of the four candidates using linguistic term set $S^9$}\label{tab:15}
\centering
\renewcommand{\arraystretch}{1.5}\small
\begin{tabular}{p{1cm}p{14cm}}
  \toprule
$G_i$  & Collective assessments  \\\midrule
  $G_1$&  $\{\langle s_{1}^{9},0.011\rangle,\langle s_{2}^{9},0.011\rangle, \langle s_{3}^{9},0.011\rangle,\langle s_{4}^{9},0.068\rangle, \langle s_{5}^{9},0.226\rangle,\langle s_{6}^{9},0.452\rangle, \langle s_{7}^{9},0.075\rangle, \langle s_{8}^{9},0.146\rangle\}$\\
  $G_2$&   $\{\langle s_{2}^{9},0.026\rangle, \langle s_{3}^{9},0.017\rangle,\langle s_{4}^{9},0.248\rangle, \langle s_{5}^{9},0.113\rangle,\langle s_{6}^{9},0.337\rangle, \langle s_{7}^{9},0.109\rangle, \langle s_{8}^{9},0.150\rangle\}$\\
  $G_3$&  $\{\langle s_{0}^{9},0.026\rangle,\langle s_{2}^{9},0.122\rangle, \langle s_{3}^{9},0.078\rangle,\langle s_{4}^{9},0.2370\rangle, \langle s_{5}^{9},0.098\rangle,\langle s_{6}^{9},0.162\rangle, \langle s_{7}^{9},0.084\rangle, \langle s_{8}^{9},0.200\rangle\}$\\
  $G_4$&  $\{\langle s_{2}^{9},0.055\rangle, \langle s_{3}^{9},0.039\rangle,\langle s_{4}^{9},0.106\rangle, \langle s_{5}^{9},0.105\rangle,\langle s_{6}^{9},0.191\rangle, \langle s_{7}^{9},0.119\rangle, \langle s_{8}^{9},0.385\rangle\}$\\
 \bottomrule
\end{tabular}
\end{table}

\section{Conclusions}\label{sec:7}

In this paper, a new linguistic computational model has been developed to deal with multi-granular linguistic distribution assessments for its application to large-scale MAGDM problems with linguistic information.

First, different distance measures and a new ranking method are developed to improve the management of linguistic distribution assessments.

Second, the relationship between a linguistic 2-tuple and a linguistic distribution assessment is investigated. To manage multi-granular linguistic distribution assessments, a new linguistic computational model is then developed based on the ELH model and the transformation formulae between a linguistic 2-tuple and a linguistic distribution assessment, which not only can be used to fuse multi-granular linguistic distribution assessments, but also can provide interpretable aggregate linguistic results to decision makers.

Third, an approach to large-scale MAGDM with multi-granular linguistic information is proposed based on the new linguistic computational model. The proposed approach uses linguistic distribution assessments to represent decision makers' assessments, which keeps the maximum information elicited by decision makers of the group in initial stages of the decision process and can provide more information about the collective assessments over alternatives.

Our future research will study the consensus reaching process for large-scale MAGDM problems with multi-granular linguistic information based on the developed model. Moreover, the hesitant fuzzy linguistic term sets proposed by Rodr\'{\i}guez \emph{et al.} \cite{Rodriguez12tfs} have received more and more attention from scholars \cite{Liu14ins,Liao14ins,weiC2015}. It will also be interesting to analyze the relationship between linguistic distribution assessments and the hesitant fuzzy linguistic term sets in the future.

\bibliographystyle{IEEEtran}
\bibliography{ieee}

\end{document}